\documentclass[sigconf]{aamas}

\usepackage{balance} 
\usepackage[linesnumbered,ruled,vlined]{algorithm2e}

\setcopyright{none}
\acmConference[]{}
\copyrightyear{ }
\acmYear{ }
\acmDOI{}
\acmPrice{}
\acmISBN{}

\title[AAMAS-2025 Formatting Instructions]{Truthful mechanisms for linear bandit games with private contexts}

\author{Yiting Hu}
\affiliation{
  \institution{Singapore University of Technology and Design}
  \country{Singapore}}
\email{yiting_hu@mymail.sutd.edu.sg}

\author{Lingjie Duan}
\affiliation{
  \institution{Singapore University of Technology and Design}
  \country{Singapore}}
\email{lingjie_duan@sutd.edu.sg}

\thanks{This work is to appear in AAMAS 2025. This work is also supported in part by the Ministry of Education, Singapore, under its Academic Research Fund Tier 2 Grant with Award no. MOE-T2EP20121-0001; in part by SUTD Kickstarter Initiative (SKI) Grant with no. SKI 2021\_04\_07; and in part by the Joint SMU-SUTD Grant with no. 22-LKCSB-SMU-053.}
\setcopyright{none}
\renewcommand\footnotetextcopyrightpermission[1]{}
\begin{abstract}
The contextual bandit problem, where agents arrive sequentially with personal contexts and the system adapts its arm allocation decisions accordingly, has recently garnered increasing attention for enabling more personalized outcomes. However, in many healthcare and recommendation applications, agents have private profiles and may misreport their contexts to gain from the system. For example, in adaptive clinical trials, where hospitals sequentially recruit volunteers to test multiple new treatments and adjust plans based on volunteers' reported profiles such as symptoms and interim data, participants may misreport severe side effects like allergy and nausea to avoid perceived suboptimal treatments. We are the first to study this issue of private context misreporting in a stochastic contextual bandit game between the system and non-repeated agents. We show that traditional low-regret algorithms, such as UCB family algorithms and Thompson sampling, fail to ensure truthful reporting and can result in linear regret in the worst case, while traditional truthful algorithms like explore-then-commit (ETC) and \(\epsilon\)-greedy algorithm incur sublinear but high regret. We propose a mechanism that uses a linear program to ensure truthfulness while minimizing deviation from Thompson sampling, yielding an $O(\ln T)$ frequentist regret. Our numerical experiments further demonstrate strong performance in multiple contexts and across other distribution families.
\end{abstract}


\newcommand{\BibTeX}{\rm B\kern-.05em{\sc i\kern-.025em b}\kern-.08em\TeX}

\begin{document}


\pagestyle{fancy}
\fancyhead{}


\maketitle 


\section{Introduction}\nocite{abramowitz1948handbook}

The contextual bandit problems have received increasing attention over the past decade, beginning with Auer’s introduction of the concept \cite{auer2002using,chu2011contextual,abbasi2011improved, agrawal2013thompson, abeille2017linear}. In the contextual bandit model, an arbitrary set of observable actions is available at each time step, and the reward for each action is determined by an unknown parameter shared across all actions. The contextual bandit excels in making personalized decisions by using contextual information to select the best possible actions. This allows for more efficient learning and better adaptation to dynamic environments compared to traditional bandit models.

However, these models do not align with scenarios involving private contexts and fail to capture the challenges posed by private information. In the new stochastic bandit problem involving private contexts \cite{lazaric2009hybrid, goldenshluger2013linear,rakhlin2016bistro, bastani2021mostly}, at each time step, a new agent arrives, reports her private context, and the system selects
one of the available $K$ arms, where each arm is associated with a different unknown parameter. The agent then receives a stochastic reward based
on the system’s chosen action, after which she leaves. This scenario is common in applications like clinical trials and online recommendations, where agents may strategically misreport private contexts to maximize single-round personal rewards. For example, in adaptive clinical trials of phase 2 INSIGHT trial where hospitals test treatments for glioblastoma based on patients' symptoms and medical history, some patients may misreport side effect histories like allergies or anemia to avoid the less-established abemaciclib treatment \cite{FDA_Abemaciclib_2017, rahman2023inaugural}. On online platforms like Netflix or Amazon, many users prefer recommendations based on popular choices or expert curation \cite{growthsetting_netflix_ai}.

In this new stochastic contextual linear bandit problem with private contexts, previous works assume observable and public contexts and do not consider the agents' strategic behavior to game the system. The conflict between the system's long-term reward and the individual's immediate reward in the multi-armed bandit problem has been studied for the past decade \cite{kremer2014implementing,mansour2020bayesian,mansour2022bayesian,sellke2021price,hu2022incentivizing,immorlica2018incentivizing,sellke2023incentivizing}. Kremer et al. \cite{kremer2014implementing} initiate the research within a Bayesian exploration framework, introducing a recommendation mechanism for incentivizing exploration with deterministic rewards. Mansour et al. \cite{mansour2020bayesian} further develop the problem to the stochastic rewards. Sellke and Slivkins \cite{sellke2021price} first prove that Thompson sampling algorithm can be naturally incentive-compatible (IC) if provided with sufficient initial samples. Then Hu et al. \cite{hu2022incentivizing} and Simchowitz and Slivkins \cite{sellke2023incentivizing} extend this result to the combinatorial and linear bandit problems. Beyond recommendation mechanisms, Immorlica et al. \cite{immorlica2018incentivizing} apply selective disclosure of historical information to encourage exploration. Simchowitz and Slivkins \cite{simchowitz2024exploration} also study this problem in reinforcement learning. These works assume that the system has the full information about agents for the recommendation, then the problem is to design the IC mechanisms that ensure agents follow the recommendation. In contrast, our system needs agents to report their private contexts, where the system lacks context information, and agents may strategically misreport their private contexts, rendering these methods ineffective for the problem addressed in this paper.


Our main contributions are summarized as follows:
\begin{itemize}
\item We are the first to study agents' strategic context misreporting to maximize their one-time individual expected rewards in the new contextual bandit problem. We demonstrate that existing algorithms perform poorly under misreporting. Specifically, existing truthful algorithms, such as the greedy and explore-then-commit (ETC) methods, suffer from relatively high regret, while low-regret algorithms like UCB family algorithms and Thompson sampling exhibit a regret of \(O(T)\) under strategic misreporting.

 \item We propose a truthful mechanism based on Thompson sampling algorithm which guarantees that agents have no incentive to misreport their contexts. We prove that our algorithm achieves a frequentist regret upper bound of \(O(\ln T)\) in the Bayesian contextual linear bandit setting. Additionally, our experiments show that the mechanism has sublinear regret when applied to multiple contexts and across some other sub-Gaussian distributions.
\end{itemize}


\subsection{Related work}

Stochastic contextual linear bandit algorithms can be categorized into deterministic algorithms, which make deterministic choices, and stochastic algorithms, which maintain a probability distribution among arms for selection. When facing agents' context misreporting, deterministic algorithms in which the choice depends on context cannot ensure truthful reporting in exploration because the resulting arm choices are predictable. Therefore, deterministic low-regret algorithms like the UCB family \cite{chu2011contextual, abbasi2011improved} suffer from linear regret in the worst-case scenario as shown in Section \ref{section3} in this paper. Another deterministic algorithm, the Explore-Then-Commit (ETC) algorithm, does not rely on agents' context but is inefficient, incurring a relatively high regret of order \(O(T^{2/3})\) \cite{lattimore2020bandit}. For stochastic algorithms, the \(\epsilon^t\)-greedy algorithm is truthful as its exploration probability is independent of the context, but it also incurs a regret order of \(O(T^{2/3})\) \cite{banditsintroduction}.

Another stochastic low-regret algorithm is Thompson sampling, which was first adapted by Agrawal and Goyal in \cite{agrawal2013thompson} for the contextual linear bandits problem. Abeille and Lazaric in \cite{abeille2017linear} further improve the frequentist regret of linear Thompson sampling. Thompson sampling is also widely applied to Bayesian bandit problems, as it naturally leverages posterior distributions. Russo and Van Roy \cite{russo2014learning, russo2016information} provide the Bayesian regret upper bound for Thompson sampling, with frequentist regret serving as an upper bound on Bayesian regret. However, we will show in Section \ref{section3} that Thompson sampling is not truthful and still suffer from linear regret under misreporting behavior. 

Regarding context misreporting behavior, Buening et al. examine this phenomenon in a different problem in \cite{buening2024strategic}. Their work considers arms as repeated strategic entities that manipulate rewards to increase their chances of being chosen. While they also address context misreporting behavior, their focus fundamentally differs from ours, and their approach is inapplicable to our problem, as our agents are non-repeated and myopic.

\section{Problem formulation}\label{section2}
We consider the Bayesian contextual linear bandit model in \cite{bastani2021mostly}. There are \( K \) arms in the set \([K]=\{1,\dots,K\}\), each associated with an unknown, fixed \(d\)-dimensional hidden parameter \(\theta_k \in \mathbb{R}^d\). These parameters \(\{\theta_k\}_{k\in[K]}\) are unknown to both the system and the agents but are drawn from a known prior distribution \(\mathcal{P}_k: \mathbb{R}^d \to \mathbb{R}\). The prior distributions \(\mathcal{P}_k\) for any $k\in[K]$ are common knowledge for both the system and the agents. We define the prior mean of each \(\mathcal{P}_k\) as \(\mu_k \in \mathbb{R}^d\) and the covariance matrix as \(V_k \in \mathbb{S}^{d}_+\).

At each time \(t \in \{1, \dots, T\}\), a new agent arrives with a private context \(x_t \in \mathcal{X}\), where \(\mathcal{X} = \{\chi_1, \dots, \chi_N\} \subsetneq \mathbb{R}^d\) is a set of finite $N$ contexts. We assume that each \(x_t\) takes the value \(\chi_n\in\mathcal{X}\) with a known probability \(\beta_n > 0\), where \(\sum_{n=1}^N \beta_n = 1\) \footnote{Our mechanism also works when the arrival probability \(\beta_n\) is unknown. We can initialize the mechanism with a uniform discrete context distribution and adjust \(\beta_n\) as we learn and update it throughout the process.}. The system first provides its history \(\mathcal{F}_t\), which includes previously observed contexts, actions, and rewards, as well as the arm selection policy \(\pi(x, \mathcal{F}_t)\) for all contexts $x \in \mathcal{X}$ which defines a probability distribution over the \(K\) arms, to the agent. After receiving this information, the agent reports the context \(x'_t \in \mathcal{X}\) to the system. Based on the reported context \(x'_t\), the system selects an arm \(a_t \in [K]\) and only observes the corresponding reward \(r_t = x_t^\top \theta_{a_t} + \eta_t\), where \(\eta_t\) is a zero-mean random variable. The system then updates the posterior estimation of the chosen arm $a_t$ to $\hat{\theta}^{t+1}_{a_t}$, and the posterior distribution to \(\mathcal{P}_{a_t}^{t+1}=\mathcal{P}_{a_t}(\cdot | \mathcal{F}_{t+1})\).

We begin by considering Gaussian priors and Gaussian noise to provide a clearer illustration of the problem and to facilitate the analysis of frequentist regret, as in \cite{agrawal2013thompson,agrawal2017near}. Specifically, we assume $\mathcal{P}_k\sim \mathcal{N}(\mu_k, V_k)$ and $\eta_t\sim\mathcal{N}(0, 1)$. This assumption allows for closed-form updates of \(\hat{\theta}_k^t\) and \(\mathcal{P}_{a_t}^t\), making the analysis more tractable. Still, our mechanism is applicable to any family of prior and noise distributions. In Section \ref{section7}, we use simulations to demonstrate that the mechanism design in Section \ref{section4} achieves good regret performance under other sub-Gaussian distributions. Under Gaussian priors and Gaussian noise, the posterior distribution is \(\mathcal{P}_k^t(\cdot) \sim \mathcal{N}(\hat{\theta}^t_k, \hat{V}^t_k)\), where \(\hat{\theta}_k^t\) and \(\hat{V}^t_k\) are updated as follows:
\begin{align}
    &\hat{V}_k^t=\bigg(V_k^{-1}+\sum_{\tau\in\mathcal{T}_k^t}x_\tau x_\tau^\top\bigg)^{-1},\hat{\theta}_k^t=\hat{V}_k^t\bigg(V_k^{-1}\mu_k+\sum_{\tau\in\mathcal{T}_k^t}x_\tau r_\tau\bigg)\label{thetakt},
\end{align}
where $\mathcal{T}_k^t$ denotes the set of time steps when the system chooses arm $k$ before time $t$.

We now formulate the objectives of both the agents and the system. Given the arm choice policy $\pi(x,\mathcal{F}_t)$ for any $x\in \mathcal{X}$ provided by the system, the agent arriving with context $x_t$ chooses to report the context \(x'_t\) that maximizes her expected reward, expressed as:
\begin{align}
x'_t = \arg\max_{x \in \mathcal{X}} {x_t}^\top \Theta^t {\pi}(x,\mathcal{F}_t),\label{agentReward}
\end{align}
where \(\Theta^t = [\hat{\theta}_1^t, \dots, \hat{\theta}_K^t]\) represents the matrix of posterior estimates at time \(t\). Based on Eq. (\ref{agentReward}), we then present the definition of the truthful mechanism as follows:

\begin{definition}
A mechanism is considered truthful in our Bayesian contextual linear bandit problem if no agent can increase her expected reward by misreporting her true context at any time step. Formally, for any history $\mathcal{F}_t$ and any pair of distinct contexts \(x_t\) and \(x'_t\) in \(\mathcal{X}\), the following condition holds at each time step \(t \in [T]\):
\begin{align}
    x_t^\top \Theta^t {\pi}(x_t,\mathcal{F}_t) \geq {x_t}^\top \Theta^t {\pi}(x'_t,\mathcal{F}_t).\label{condition}
\end{align}
\end{definition}
When a context \(x \in \mathcal{X}\) has an incentive to misreport as another context \(x' \in \mathcal{X}\), we say that contexts \(x\) and \(x'\) have conflict.

Conversely, the system’s objective is to maximize the cumulative reward, or equivalently, to minimize the expected total regret by choosing $a_t$ for each time step $t$ in the truthful mechanism $\pi(\cdot)$. Let \(a_t^*\) denote the optimal arm for the agent arriving at time \(t\). The expected total regret is:
\begin{align}
    \mathbb{E}[\mathcal{R}(T)] = \mathbb{E}\left[\sum_{t=1}^T \left(x_t^\top \theta_{a_t^*} - x_t^\top \theta_{a_t}\right)\right]. \label{regret}
\end{align}

Building on the truthful mechanism defined above, we will next analyze the behavior of existing bandit algorithms under misreporting.


\section{Performance analysis of existing algorithms under misreporting}\label{section3}
In this section, we present comprehensive studies on the performances of existing deterministic and stochastic algorithms under agents' possible misreporting. 

\subsection{Deterministic Algorithms} 

In deterministic algorithms, the algorithm selects one of the \(K\) arms based on the history with probability 1 at each time step. A well-known class of deterministic algorithms for the contextual linear bandit problem is the UCB family, including LinUCB \cite{chu2011contextual} and OFUL \cite{abbasi2011improved}, which select arms at each time step based on upper confidence bounds. However, UCB family algorithms are vulnerable to misreporting because their allocation is predictable, allowing agents to easily manipulate their reported context to favor the currently optimal arm, which can lead to linear regret in the worst-case scenario. 

In contrast, the deterministic Explore-Then-Commit (ETC) algorithm, which first operates in a round-robin exploration phase before switching to a purely greedy strategy, is truthful because its decisions are independent of personal contexts, making agents’ context reporting irrelevant to the algorithm’s choice. However, the ETC algorithm incurs a relatively high regret of \(O(T^{2/3})\) \cite{lattimore2020bandit}.

\subsection{Stochastic Algorithms} 

In stochastic algorithms, the algorithm maintains a probability distribution over the arms at each time step and selects an arm according to this distribution. The \(\epsilon^t\)-greedy algorithm, which selects the greedy arm with probability \(\epsilon^t\) and chooses an arm uniformly at random with probability \(1 - \epsilon^t\), is truthful since \(\epsilon^t\) can be set so that the selection probability is independent of the contexts. However, it also suffers from a relatively high regret of \(O(T^{2/3})\) \cite{banditsintroduction}. Next, we consider Thompson sampling algorithm \cite{agrawal2013thompson}. 

In Thompson sampling, given the reported context \(x'_t\) at each time step, the system draws a sample \(\tilde{\theta}_k^t\) from the posterior distribution of \({x'_t}^\top \theta_k\), denoted as \({\mathcal{P}}^t_k(\cdot | x'_t) : \mathbb{R} \to \mathbb{R}\), and then selects arm \(\bar{a}_t = \arg\max_k \tilde{\theta}^t_k\). Note that \(\mathcal{P}_k^t(\cdot)\) represents the posterior distribution of \(\theta_k\) at time $t$, whereas \({\mathcal{P}}^t_k(\cdot|x_t')\) is the posterior distribution of \(x_t'^\top \theta_k\). The process of first sampling \(\theta_k\) from \({\mathcal{P}}^t_k\) and then multiplying it by \(x_t'\) yields the same result as directly sampling from \({\mathcal{P}}^t_k(\cdot|x_t')\) when assuming Gaussian prior and noise. Therefore, the distribution of arm selection $\mathbf{p}^t(x_t')=({p}^t_1(x_t'), \dots,{p}^t_K(x_t'))$ in Thompson sampling, which is also the policy \(\pi(x'_t, \mathcal{F}_t)\) at time \(t\) is:

\begin{align}
    p_k^t(x'_t) = \int \cdots \int_{\tilde{\theta}_k^t \geq \tilde{\theta}_j^t, \, j \in [K], \, j \neq k} d {\mathcal{P}}^t_1(\tilde{\theta}_1^t | x'_t) \cdots d {\mathcal{P}}_K(\tilde{\theta}^t_K | x'_t)\label{TSP}.
\end{align}

We can demonstrate that Thompson sampling is not truthful through a simple example in Fig. \ref{exam1}. At time $t$, given the posterior estimation \(\hat{\theta}^t_k\) for all \(k \in [3]\) and the expected arm choice policy ${\pi}(x_t', \mathcal{F}_t)=(p^t_1(x_t'), p^t_2(x_t'), p^t_3(x_t'))$ in (\ref{TSP}), the resulting expected parameter \(\Theta^t \mathbf{p}^t(x_t')\) lies within the convex hull \(\text{conv}(\hat{\theta}^t_1, \hat{\theta}^t_2, \hat{\theta}^t_3)\). Therefore, Thompson sampling can be seen as a mapping from any context \(x_t' \in \{\chi_1,\chi_2\}\) to $\Theta^t \mathbf{p}^t(x_t')$ within the convex hull. In the Thompson sampling mapping of $\Theta^t \mathbf{p}^t(\chi_1)$ and $\Theta^t \mathbf{p}^t(\chi_2)$ in Fig. \ref{exam1}, context \(\chi_2\) yields a higher inner product with \(\Theta^t \mathbf{p}^t(\chi_1)\) than with \(\Theta^t \mathbf{p}^t(\chi_2)\). Consequently, if \(\chi_2\) arrives at time $t$, it will misreport as \(\chi_1\).

\begin{figure}
    \centering
    \includegraphics[width=0.3\linewidth]{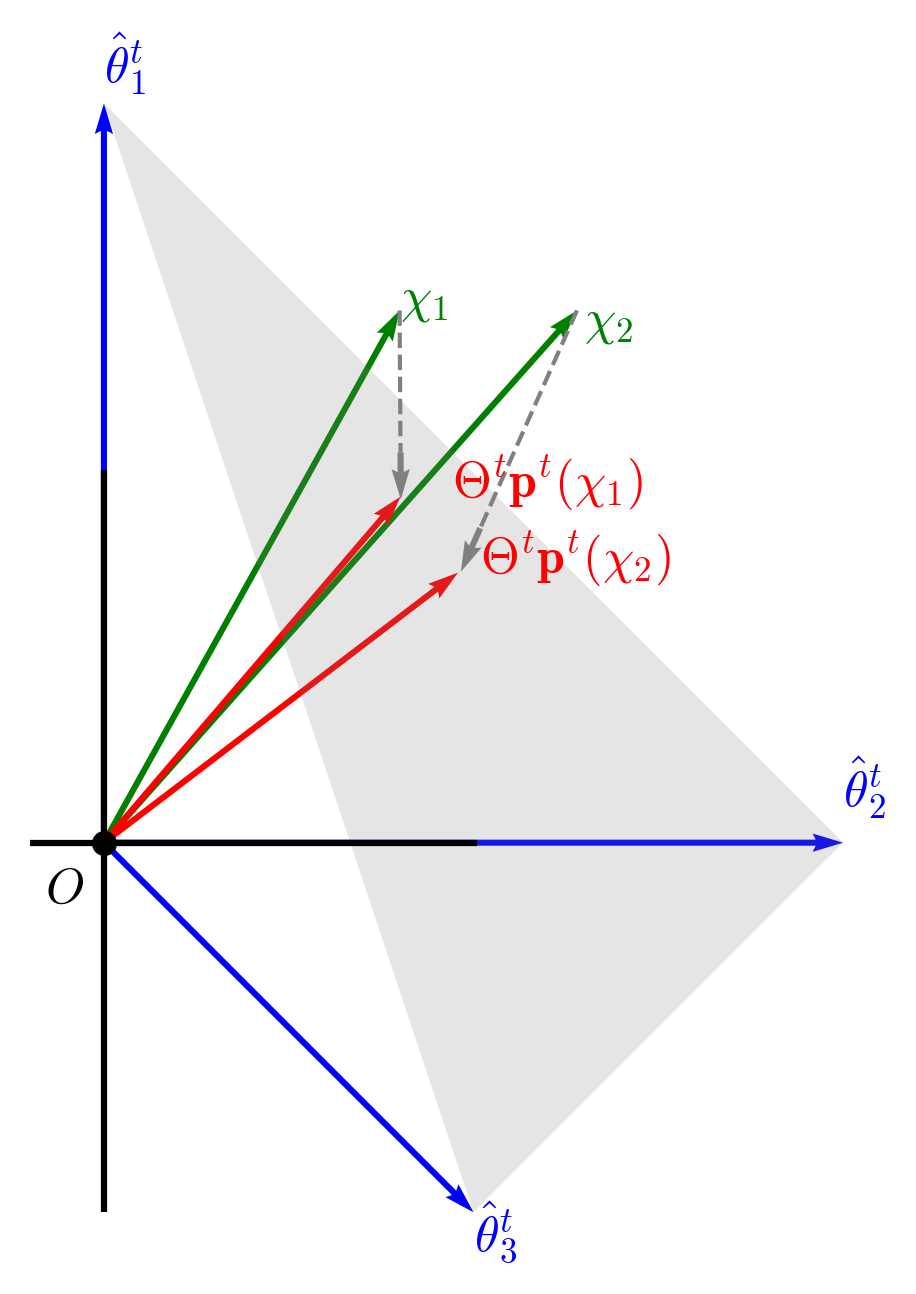}
    \caption{Geometric illustration of context \(\chi_2\)'s incentive to misreport as \(\chi_1\).}
    \label{exam1}
\end{figure}
We further prove the regret of Thompson sampling under the context misreporting.



\begin{lemma}\label{TSL}
Thompson sampling algorithm cannot ensure truthful context reporting and results in linear regret $O(T)$ in the worst case.
\end{lemma}

\begin{proof}[Sketch of Proof]
We prove the lemma by constructing an example in which, given a specific prior \(\mathcal{P}_1 \times \cdots \times \mathcal{P}_K\), one of the contexts has an incentive to misreport. Then, under a certain context arrival distribution \(\{\beta_n\}_{n \in [N]}\), we find a positive probability that the algorithm fails to identify the true optimal arm for this context throughout the process and ultimately converge to a suboptimal arm for this context. The complete proof of Lemma \ref{TSL} can be found in Appendix \ref{APTSL}.
\end{proof}
Given that existing deterministic and stochastic algorithms fail due to either untruthfulness or inefficiency, there is a need to design more effective, truthful mechanisms.

\section{Truthful Thompson sampling mechanism}\label{section4}
In this section, we introduce our truthful-Thompson sampling (truthful-TS) mechanism for contextual linear bandits to ensure truthful reporting under Thompson sampling. Our objective is to determine a new probability distribution \(\mathbf{q}^t(\chi_n)\) across the \(K\) arms at each time \(t\) for any \(n \in [N]\) that guarantees truthfulness. We derive \(\{\mathbf{q}^t(\chi_n)\}_{n \in [N]}\) from the Thompson sampling probabilities \(\{\mathbf{p}^t(\chi_n)\}_{n \in [N]}\) in Eq. (\ref{TSP}), aiming to keep \(\{\mathbf{q}^t(\chi_n)\}_{n \in [N]}\) as close as possible to \(\{\mathbf{p}^t(\chi_n)\}_{n \in [N]}\). To achieve this, we formulate a linear optimization problem (LP) at each time $t$, given by:
 \begin{align}
    \text{minimize }&\max_{n\in [N]}(||\mathbf{p}^t(\chi_n) - \mathbf{q}^t(\chi_n)||_{\infty})\nonumber\\
    \text{s.t. }&\chi_i^\top\Theta^t(\mathbf{q}^t(\chi_i)-\mathbf{q}^t(\chi_j))\geq0,\quad\forall i\neq j,i,j\in[N]\nonumber\\
    &\sum_{n\in [N]}\beta_n\mathbf{q}^t(\chi_n)=\sum_{n\in [N]}\beta_n\mathbf{p}^t(\chi_n),\nonumber\\
    & {q}^t_{1}(\chi_n)+\dots+{q}^t_{K}(\chi_n)=1,\quad n\in[N]\nonumber\\
    &{q}^t_{k}(\chi_n)\geq 0\quad k\in[K], \quad n\in[N].\label{LP1}
\end{align}
In (\ref{LP1}), the objective is to minimize the maximum difference between any \(p_k^t(\chi_n)\) and \(q_k^t(\chi_n)\) across all possible contexts, keeping the new distribution as aligned as possible with the Thompson sampling probabilities. The first constraint ensures that the agent with private context \(\chi_i\) cannot obtain a higher expected reward \(\chi_i^\top\Theta^t\mathbf{q}^t(\chi_j)\) by misreporting as \(\chi_j\) than the reward \(\chi_i^\top\Theta^t\mathbf{q}^t(\chi_i)\) by truthfully reporting. The second constraint ensures that the weighted average of choosing arm \(k\) across all contexts, based on the arrival probability \(\beta_n\) of context $\chi_n$, remains the same as the weighted average of Thompson sampling probability \(p_k^t(\cdot)\). The third and fourth constraints ensure that the solution \(\mathbf{q}^t(x)\) for each \(x \in \mathcal{X}\) forms a valid probability distribution. 

Based on (\ref{LP1}), we present our truthful-TS mechanism in Mechanism \ref{A1}. To demonstrate the feasibility of our mechanism, we first need to show that the LP in problem (\ref{LP1}) has a feasible and convergent solution, where the probability of selecting the optimal arm converges to 1. We can easily construct such a feasible solution. Define the conflict clustering \(\mathbf{C} = \{C_1, \dots, C_j\}\), where each $C\in \mathbf{C}$ is a subset of contexts and each context \(x \in \mathcal{X}\) belongs to exactly one cluster \(C \in \mathbf{C}\). Within each cluster, every context has a conflict with at least one other context in the same cluster and has no conflicts with contexts in any other clusters. Let \(\mathcal{C}\) be a mapping from any context \(x \in \mathcal{X}\) to its respective cluster, such that \(\mathcal{C}: \mathcal{X} \rightarrow \mathbf{C}\). Then, we can construct a feasible solution for (\ref{LP1}) as follows:
\begin{align}
    \mathbf{q}^t(x) = \sum_{\chi_i\in \mathcal{C}(x)}\frac{\beta_i\mathbf{p}^t(\chi_i)}{\sum_{\chi_i\in \mathcal{C}(x)} \beta_i},\forall x\in\mathcal{X}.\label{FC}
\end{align}
However, a better solution for \(\mathbf{q}^t(x)\) can be obtained by solving (\ref{LP1}), under the condition in Lemma \ref{feasible}.

\begin{algorithm}
\caption{Truthful-Thompson sampling mechanism}\label{A1}
\DontPrintSemicolon
\For{$t=1$ to $T$}{
 Calculate the arm choice distribution \(\mathbf{p}^t(x)\) for each \(x \in \mathcal{X}\) in Thompson sampling algorithm by Eq. (\ref{TSP}).\\ 
Solve \(\{\mathbf{q}^t(x)\}_{x \in \mathcal{X}}\) in (\ref{LP1}), using \(\mathcal{X}\), \(\{\mathbf{p}^t(x)\}_{x \in \mathcal{X}}\), $\{\hat{\theta}^t_k\}_{k\in[K]}$ and $\{\hat{V}^t_k\}_{k\in[K]}$ as inputs.\\
 Provide the solution \(\{\mathbf{q}^t(x)\}_{x \in \mathcal{X}}\), the posterior estimates \(\{\hat{\theta}^t_k\}_{k\in[K]}\) and \(\{\hat{V}^t_k\}_{k\in[K]}\) to the agent arriving at time \(t\).\\ 
 Observe the context \(x'_t\) reported by the agent.\\
 Choose arm \(a_t\) according to the probability distribution \(\mathbf{q}^t(x'_t)\).\\ 
 Observe the reward \(r_t\), then update \(\hat{V}_{a_t}^t\) and \(\hat{\theta}_{a_t}^t\) based on Eq. (\ref{thetakt}).} 
\end{algorithm}

\begin{lemma}\label{feasible}
Problem (\ref{LP1}) must have a feasible and convergent solution as in (\ref{FC}). Additionally, as long as $\chi_n^\top (\Theta^t(\cdot, 1:K-1)-\hat{\theta}^t_K\mathbf{1}^\top)$ are linearly independent across all $n\in[N]$, i.e.,
\begin{align}
    \sum_{n\in[N]}\lambda_n\chi_n^\top (\Theta^t(\cdot, 1:K-1)-\hat{\theta}^t_K\mathbf{1}^\top)\neq 0,\;\forall (\lambda_1,\dots,\lambda_N)\neq \mathbf{0},\label{conditionLP}
\end{align}
the feasible solution space of problem (\ref{LP1}) has a non-zero measure around $\{\mathbf{q}^t(x)\}_{x\in\mathcal{X}}$ in (\ref{FC}) with infinitely many possible solutions.
\end{lemma}
\begin{proof}[Sketch of proof]
It is straightforward to observe that setting \(\mathbf{q}^t(x)\) as the weighted average of \(\mathbf{p}^t(x)\) within each subset of conflicting contexts, with weights proportional to \(\beta_n\), yields a feasible solution. To prove the second part of the lemma, we start by noting that the feasible solution of (\ref{LP1}) must take the form \(\mathbf{q}^t(\chi_n) = \sum_{i \in [N]} \beta_i \mathbf{p}^t(\chi_i) + \xi_n\), where \(\sum_{n \in [N]} \beta_n \xi_n = 0\). By substituting \(\mathbf{q}^t(\chi_n) = \sum_{i \in [N]} \beta_i \mathbf{p}^t(\chi_i) + \xi_n\) into (\ref{LP1}) and redefining the variables in terms of \(\xi_n\) for all \(n \in [N]\), we then reformulate problem (\ref{LP1}) into an equivalent form as follows:
\begin{align}
    \min \max_{n\in[N]}&\left\| {\sum_{i \in [N]}\beta_i\mathbf{p}^t(\chi_i)} -\mathbf{p}^t(\chi_n)+ \xi_n\right\|_{\infty} \nonumber \\
    \text{s.t.} &\quad \chi_i^\top \Theta^t (\xi_i - \xi_j) \geq 0, \quad \forall i\neq j, \ i,j\in[N] \nonumber \\
    &\quad \xi_{n,1} + \dots + \xi_{n,K} = 0, \quad \forall n\in[N] \nonumber \\
    &\quad \beta_1\xi_{1,k} + \dots + \beta_N\xi_{N,k} = 0, \quad \forall k\in[K] \nonumber \\
    &\quad 0 \leq \sum_{i}\beta_i{p}^t_k(\chi_i) + \xi_{n,k} \leq 1, \quad \forall n\in[N], k\in[K]. \label{LP2-}
\end{align}
We can reformulate the first three constraints of (\ref{LP2-}) as a convex cone given by
\begin{align*}
\begin{pmatrix}
    \chi_1^\top (\Theta^t(\cdot, 1:K-1) - \hat{\theta}^t_K\mathbf{1}^\top) \otimes A_1^\top \\
    \vdots \\
    \chi_N^\top (\Theta^t(\cdot, 1:K-1) - \hat{\theta}^t_K\mathbf{1}^\top) \otimes A_N^\top
\end{pmatrix}
\text{vec}(\mathcal{E}^\top) \geq 0,
\end{align*}
where $\mathcal{E}$ represents the first $K-1$ rows and first $N-1$ columns of matrix $(\xi_1,\dots,\xi_N)$, and \(A_n\) is a constant matrix constructed for each \(n \in [N]\). We show that this convex cone has a non-zero measure when all vectors \( \chi_n^\top(\Theta^t(\cdot, 1:K-1) - \hat{\theta}^t_K \mathbf{1}^\top) \) for \(n \in [N]\) are linearly independent, based on the properties of the constructed matrix \(A_n\). Furthermore, since the distribution \(\mathbf{p}^t(x)\) for any \(x \in \mathcal{X}\) lies within the interior of the simplex \(\Delta^K\), representing all probability distributions over \(K\) arms, we can construct a feasible solution space with a non-zero measure. For the complete proof of this lemma, please refer to Appendix \ref{APfeasible}.
\end{proof}

When Eq. (\ref{conditionLP}) is satisfied, we can modify the first constraint of problem (\ref{LP1}) to \(\chi_i^\top \Theta^t (\mathbf{q}^t(\chi_i) - \mathbf{q}^t(\chi_j)) \geq \epsilon\), where \(\epsilon\) is a sufficiently small positive value to ensure that truthful reporting becomes a strictly dominant strategy for agents. Furthermore, when the feasible space has a non-zero measure, it can yield an improved optimal value for problem (\ref{LP1}) compared to $\{\mathbf{q}^t(x)\}_{x\in\mathcal{X}}$ in (\ref{FC}).

However, solving the linear program in (\ref{LP1}) incurs a complexity higher than \(O((KN)^{2+1/6})\) \cite{cohen2021solving}, whereas identifying (\ref{FC}) has a lower complexity of \(O(KN^2)\), which arises from the process of identifying conflict clusters. For larger \( K \) and \( N \), we can improve efficiency of Mechanism \ref{A1} by using \(\{\mathbf{q}^t(x)\}_{x \in \mathcal{X}}\) from (\ref{FC}) as a substitute for solving (\ref{LP1}).




\section{Regret Analysis Under the Same Optimal Arm for two contexts}\label{section5}
In this section, we analyze the regret performance of our truthful mechanism in the simple but fundamental case of two contexts. The two-context scenario is
common in real-world settings. For example, online platforms often categorize users as either Mac or Windows users to tailor sales strategies \cite{adekotujo2020comparative}. Similarly, in clinical trials, hospitals categorize patients as either treatment-naive or treatment-experienced when conducting studies \cite{neale2007treatment}. For agents with two possible contexts, the scenarios are limited to either having the same or different optimal arms to misreport the other context. As \(N\) increases, the misreporting patterns become exponentially more intricate among agents, significantly complicating the regret analysis. In Section \ref{section7}, we still run simulations for multiple contexts to show similar results as presented in this section.

We focus on problem-dependent frequentist regret because misreporting behavior is influenced by the specific realization of the prior, requiring separate analysis for different cases. Specifically, we divide the realizations into two cases: (1) when the two contexts share the same optimal arm, and (2) when the two contexts have different optimal arms. The differing misreporting incentives for each case lead to major differences in regret analysis. Still, Bayesian regret can be obtained by taking the expectation over our frequentist regret, yielding the same order. We begin by analyzing the regret in the scenario where the two contexts, \(\chi_1\) and \(\chi_2\), share the same optimal arm in this section. Addressing the other scenario requires new extra techniques (see Section \ref{section6}).

When two contexts share the same optimal arm, the probability of selecting the optimal arm will eventually converge to 1 for both contexts. However, one context must converge faster than the other. Consequently, the context with a slower convergence rate may have an incentive to misreport for most of the time steps during the process. Inspired by this, we consider the two contexts collectively and derive an upper bound on the total number of suboptimal pulls for both contexts.
\begin{theorem}\label{Theorem1}
For the realization of prior, \(\{\theta_k\}_{k \in [K]}\), such that the two contexts \(\chi_1\) and \(\chi_2\) share the same optimal arm \(\alpha\), the frequentist regret of our truthful-TS mechanism in Mechanism \ref{A1} is $O(\ln T)$ to be upper bounded by
\begin{align*}
&\sum_{j\neq \alpha}\bigg(\sum_{n=1}^2 \frac{18}{\Delta_{n,j}^2}\ln \frac{T\Delta_{n,j}^2}{36}+C_{n,j}\bigg)(1+{\beta_{3-n}}/{\beta_{n}}) \max_{n=1,2} \Delta_{n,j},
\end{align*}
where \(\Delta_{n,j}=\chi_n^\top\theta_{\alpha}-\chi_n^\top\theta_j\) is the reward gap between the optimal arm \(\alpha\) and arm \(j \neq \alpha\) for context \(\chi_n\). $C_{n,j}$ is a constant for $n\in\{1,2\}$ and $j\in[K]$.
\end{theorem}
\begin{proof}
Let \(\alpha\) denote the optimal arm for both contexts. Recall that $a_t$ represents the arm chosen under our truthful-TS mechanism in Mechanism \ref{A1}, and $\bar{a}_t$ represents the arm chosen under the standard TS algorithm. We analyze the worst-case scenario where the two contexts conflict at every step, which yields an upper bound on the regret of Mechanism \ref{A1}. Indeed, for a trajectory in which the contexts have no conflict, we can bound its regret by that of standard Thompson sampling. Since the two contexts share the same optimal arm, we can decompose the regret in equation (\ref{regret}) as follows:
    \begin{align}
        \mathbb{E}[\mathcal{R}(T)]
        =& \mathbb{E}\left[\sum_{t=1}^T \left(x_t^\top \theta_{a_t^*} - x_t^\top \theta_{a_t}\right)\right]\nonumber\\
    =&\sum_{t=1}^T\sum_{n=1}^2\sum_{j\neq\alpha}\mathbb{E}[\mathbf{1}(x_t=\chi_n,a_t=j)]\Delta_{n,j}\nonumber\\
        \leq&\sum_{t=1}^T\sum_{n=1}^2\sum_{j\neq\alpha}\mathbb{E}[\mathbf{1}(x_t=\chi_n,a_t=j)]{\max_{n=1,2}\Delta_{n,j}}\nonumber\\
        =&\sum_{t=1}^T\sum_{n=1}^2\sum_{j\neq\alpha}\mathbb{E}[\beta_n\mathbb{E}[\mathbf{1}(a_t=j)|x_t=\chi_n,\mathcal{F}_t]]{\max_{n=1,2}\Delta_{n,j}}.\label{eq21}
        \end{align}
The expectation in the first equality is taken over the arrivals of contexts, the arm selections, and the observed rewards. Then, using the second constraint of our LP in (\ref{LP1}), the $\mathbb{E}[\beta_n\mathbb{E}[\mathbf{1}(a_t=j)|x_t=\chi_n,\mathcal{F}_t]]$ of (\ref{eq21}) equals the following:
        \begin{align}
\mathbb{E}[\beta_nq^t_j(\chi_n)]
        =&\mathbb{E}[\beta_np^t_j(\chi_n)]
        =\mathbb{E}[\mathbf{1}(x_t=\chi_n,\bar{a}_t=j)].\label{eq30}
        \end{align}
       Consider an alternative implementation of Mechanism \ref{A1}. At each time step $t$, if the two contexts conflict, regardless of which context actually arrives, the system first samples a virtual context \( 
\tilde{x}_t \,\in\,\{\chi_1,\chi_2\}\) according to the distribution $\{\beta_1,\beta_2\}$, and then selects an arm by running Thompson sampling on that sampled context $\tilde{x}_t$. Under this, the (\ref{eq30}) can be decomposed as
        \begin{align*}
           \mathbb{E}[\mathbf{1}(x_t=\chi_n,\tilde{x}_t=\chi_{n},\bar{a}_t=j)+\mathbf{1}(x_t=\chi_n,\tilde{x}_t=\chi_{3-n},\bar{a}_t=j)].\nonumber
        \end{align*}
Along the trajectory for the first term above, the expectation coincides exactly with standard Thompson sampling. Note that, given history \(\mathcal{F}_t\) and the sampled virtual context \(\tilde x_t\), the mechanism selects arm \(j\) solely according to its Thompson‐sampling probability for \(\tilde x_t\), rendering the choice independent of the actual context \(x_t\). Under this observation, the second term above can be rewritten as:
\begin{align}
&\mathbb{E}[\mathbb{E}[\mathbf{1}(x_t=\chi_n,\tilde{x}_t=\chi_{3-n},\bar{a}_t=j)|\mathcal{F}_t]]\nonumber\\
    =&\mathbb{E}[\mathbb{E}[\mathbf{1}(\tilde{x}_t=\chi_{3-n},\bar{a}_t=j)|\mathcal{F}_t]\mathbb{P}(x_t=\chi_n)]\nonumber\\
    =&\mathbb{E}[\mathbb{E}[\mathbf{1}(\tilde{x}_t=\chi_{3-n},\bar{a}_t=j)|\mathcal{F}_t]\mathbb{P}(x_t=\chi_{3-n})]{\beta_{n}}/{\beta_{3-n}}\nonumber\\
    =&\mathbb{E}[\mathbf{1}(x_t=\chi_{3-n},\tilde{x}_t=\chi_{3-n},\bar{a}_t=j)]{\beta_{n}}/{\beta_{3-n}}.\label{eq31}
\end{align}
By combining the equations above, the term $\sum_{t=1}^T\sum_{n=1}^2\mathbb{E}\!\bigl[\beta_n\,\mathbb{E}[1(a_t=j)\mid x_t=\chi_n,\mathcal{F}_t]\bigr]$
in (\ref{eq21}), which quantifies the expected number of pulls of the suboptimal arm \(j\), can be equivalently written as:
        \begin{align}
        \sum_{t=1}^T\sum_{n=1}^2\mathbb{E}[\mathbf{1}(x_t=\chi_n,\tilde{x}_t=\chi_n,\bar{a}_t=j)](1+{\beta_{3-n}}/{\beta_{n}}).\label{eq17}
\end{align}
In this way, we convert the regret of our truthful-TS mechanism into the regret under the Thompson sampling algorithm. Define \(\tilde{\theta}^t_{n,k}\) as the value sampled from the posterior distribution for arm \( k \) by context \( \chi_n \) at time \( t \) in the Thompson sampling algorithm. Inspired by the method from \cite{agrawal2017near}, which bounds the number of suboptimal arm pulls in Thompson sampling, we decompose Eq. (\ref{eq17}) by considering the following two events: one event $E^\mu_{n,j}(t)$ where the posterior mean estimate \(\chi_n^\top \hat{\theta}_k^t\) does not deviate significantly from the true value \(\chi_n^\top \theta_k\), and the other event $E^\theta_{n,j}(t)$ where \(\tilde{\theta}^t_{n,k}\) remains close to the posterior mean \(\chi_n^\top \hat{\theta}_k^t\) at time $t$. These events are formally defined as follows:
\begin{align*}
    &E^{\mu}_{n,j}(t): \chi_n^\top \hat{\theta}_j^t \leq \chi_n^\top {\theta}_j + \frac{\Delta_{n,j}}{3},\\
    &E^{\theta}_{n,j}(t): \tilde{\theta}_{n,j}^t \leq \chi_n^\top \hat{\theta}_j^t + \frac{\Delta_{n,j}}{3}, \quad n \in\{1,2\},\ j \in [K].
\end{align*}
Using these event realizations, we decompose equation (\ref{eq17}) and, summing over all \(t\) and \(n\), obtain:
\begin{align}
& \sum_{t=1}^T\sum_{n=1}^2\mathbb{E}[\mathbf{1}(\bar{a}_t=j,x_t=\chi_n,\tilde{x}_t=\chi_n)](1+{\beta_{3-n}}/{\beta_{n}})\nonumber\\
    =&\sum_{t=1}^T\sum_{n=1}^2\mathbb{E}[\mathbf{1}(\bar{a}_t=j,\tilde{x}_t=\chi_n)|M_n(t)=t-1](1+{\beta_{3-n}}/{\beta_{n}})\nonumber\\ = & \sum_{t=1}^T\sum_{n=1}^2 \mathbb{E}\left[(\mathbf{1}(\bar{a}_t = j, E^{\mu}_{n,j}(t), E^{\theta}_{n,j}(t))  + \mathbf{1}(\bar{a}_t = j, E^{\mu}_{n,j}(t), \overline{E^{\theta}_{n,j}(t)}) \right. \nonumber\\
    & \left. + \mathbf{1}(\bar{a}_t = j, \overline{E^{\mu}_{n,j}(t)}))\mathbf{1}(\tilde{x}_t=\chi_n) \mid M_n(t)=t-1 \right](1+{\beta_{3-n}}/{\beta_{n}}),\label{eq23}
\end{align}
where \(M_n(t)\) represent the number of agent arrivals with context \(\chi_n\) before time \(t\).

Inspired by the method in \cite{agrawal2017near}, we upper bound the three terms using the following Lemmas \ref{lemmaut1}, \ref{lemmaut2}, and \ref{lemmaut3}, respectively. The complete proofs of these lemmas can be found in Appendix \ref{APlemmaut1}, \ref{APlemmaut2} and \ref{APlemmaut3}. The upper bound of the first term directly follows from Lemma  \ref{lemmaut1}. The upper bounds of the second and third terms are obtained from Lemma \ref{lemmaut2} and Lemma \ref{lemmaut3} by setting \(\delta = \Delta_{n,j}/3\). Let \(C_{n,j}\) summarize all the constant parts in the number of pulls for arm $j$ and context $n$, then we obtain and complete the proof of Theorem \ref{Theorem1}.

\begin{lemma}\label{lemmaut1}
    In Thompson sampling with arm action $\bar{a}_t$ at time $t$, the expected total number of pulls of a suboptimal arm \( j\neq \alpha \) by context \( \chi_n\in\mathcal{X} \), together with the occurrence of events \( E^{\mu}_{n,j}(t) \) and \( E^{\theta}_{n,j}(t) \), can be upper bounded by a constant $C_{n,j}^1$ for $n\in\{1,2\}$ and $j\in[K]$:
\begin{align*}
    &\sum_{t=1}^T\mathbb{E}\left[\mathbf{1}(\bar{a}_t = j, E^{\mu}_{n,j}(t), E^{\theta}_{n,j}(t))\mid\tilde{x}_t=\chi_n,M_n(t)=t-1\right]< C_{n,j}^1.
\end{align*}

\end{lemma}
\begin{lemma}\label{lemmaut2}
In Thompson sampling with arm action $\bar{a}_t$ at time $t$, the expected total number of pulls of a suboptimal arm $j\neq \alpha$ by context $\chi_n\in\mathcal{X}$, together with the occurrence of event $\tilde{\theta}^t_{n,j}-\chi_n^\top\hat{\theta}^t_j>\delta$ for $\delta>0$, can be upper bounded as follows:
    \begin{align*}
        &\sum_{t=1}^T\mathbb{E}\bigg[\mathbf{1}(\bar{a}_t=j,\tilde{\theta}^t_{n,j}-\chi_n^T\hat{\theta}_j^t>\delta)\mid \tilde{x}_t=\chi_n,M_n(t)=t-1\bigg]\\
        \leq&  \frac{2}{\delta^2}\ln \frac{T\delta^2}{4}-\frac{1}{\chi_n^\top V_j\chi_n}+\frac{2}{\delta^2}.
    \end{align*}
\end{lemma}

\begin{lemma}\label{lemmaut3}                                      
In Thompson sampling with arm action $\bar{a}_t$ at time $t$, the expected total number of pulls of suboptimal arm $j\neq \alpha$ under any context $\chi_n\in\mathcal{X}$, together with the occurrence of event $\chi_n^\top\hat{\theta}^t_j-\chi_n^\top{\theta}_j>\delta$ for $\delta>0$, can be upper bounded as follows:
    \begin{align*}
        &\sum_{t=1}^T\mathbb{E}\bigg[\mathbf{1}(\bar{a}_t=j,\chi_n^T\hat{\theta}_j^t-\chi_n^\top\theta_{j}>\delta)\mid \tilde{x}_t=\chi_n,M_n(t)=t-1\bigg]\\
        \leq&\max\left(\left\lceil \frac{\chi_n^\top (\mu_j - \theta_j)}{\delta \chi_n^\top V_j \chi_n} - \frac{1}{\chi_n^\top V_j \chi_n} \right\rceil, 0\right)+1+\frac{\exp\bigg(\frac{\delta\chi_n^T(\mu_j-\theta_j)}{2\chi_n^TV_j\chi_n}\bigg)}{\delta^2}.
    \end{align*}
\end{lemma}

\end{proof}

Theorem \ref{Theorem1} and our proof demonstrate that our truthful-TS mechanism shares the same regret order as the Thompson sampling algorithm, indicating that our mechanism can achieve an optimal regret order while ensuring truthfulness.

\section{Regret Analysis Under Different Optimal Arms for two contexts}\label{section6}
When two contexts have different optimal arms, define the true optimal arms for \(\chi_1\) and \(\chi_2\) as \(\alpha_1\) and \(\alpha_2\), respectively, with \(\alpha_1 \neq \alpha_2\). The regret-bounding method used in the previous section, where contexts share the same optimal arm, cannot be applied here. This is because arm \(\alpha_{3-n}\) is suboptimal for context \(\chi_n\) but optimal for context \(\chi_{3-n}\), which prevents us from jointly bounding the total number of pulls of arm \(\alpha_{3-n}\) across both contexts. As a result, we must also consider the number of pulls of the optimal arm \(\alpha_{3-n}\) for context \(\chi_{3-n}\), which is not considered when bounding the regret in Thompson sampling.

First, we illustrate the misreporting incentives over time in Thompson sampling when agents have different optimal arms based on their beliefs, which provides insight into our proof approach. For illustration, we assume an ideal Thompson sampling scenario in which agents report truthfully in Fig. \ref{exam21} and \ref{exam22}. Initially, agents may have an incentive to misreport even if they have different prior optimal arms, as illustrated in Fig. \ref{exam21}. This is because, at first, the variance of the prior distribution is relatively large to encourage exploration, causing the range under the Thompson sampling mapping from \(\mathcal{X}\) to \(\text{conv}(\hat{\theta}_1^t, \hat{\theta}_2^t, \hat{\theta}_3^t)\) to be concentrated near the center of \(\text{conv}(\hat{\theta}^t_1, \dots, \hat{\theta}^t_K)\). As a result, even though the probability of selecting \(\chi_2\)’s prior optimal arm 2 is higher than that of $\chi_1$ (\({p}^t_2(\chi_2) > {p}^t_2(\chi_1)\)), \(\chi_2\) can still obtain a higher expected reward \(\chi_2^\top \Theta^t \mathbf{p}^t(\chi_1)\) by misreporting as \(\chi_1\). As the algorithm progresses and posterior variance decreases, the range of the Thompson sampling mapping shifts towards the extreme points of the simplex $\Delta^3$, as shown in Fig. \ref{exam22}, where each context ultimately achieves a higher expected reward under its own Thompson sampling distribution. Thus, when contexts have different optimal arms, the number of pulls of arm $\alpha_{3-n}$ by $\chi_n$ can be analyzed by separately considering the time steps when the contexts are in conflict and when they are not.
\begin{figure}
    \centering
    \begin{minipage}{0.18\textwidth}
        \centering
        \includegraphics[width=\textwidth]{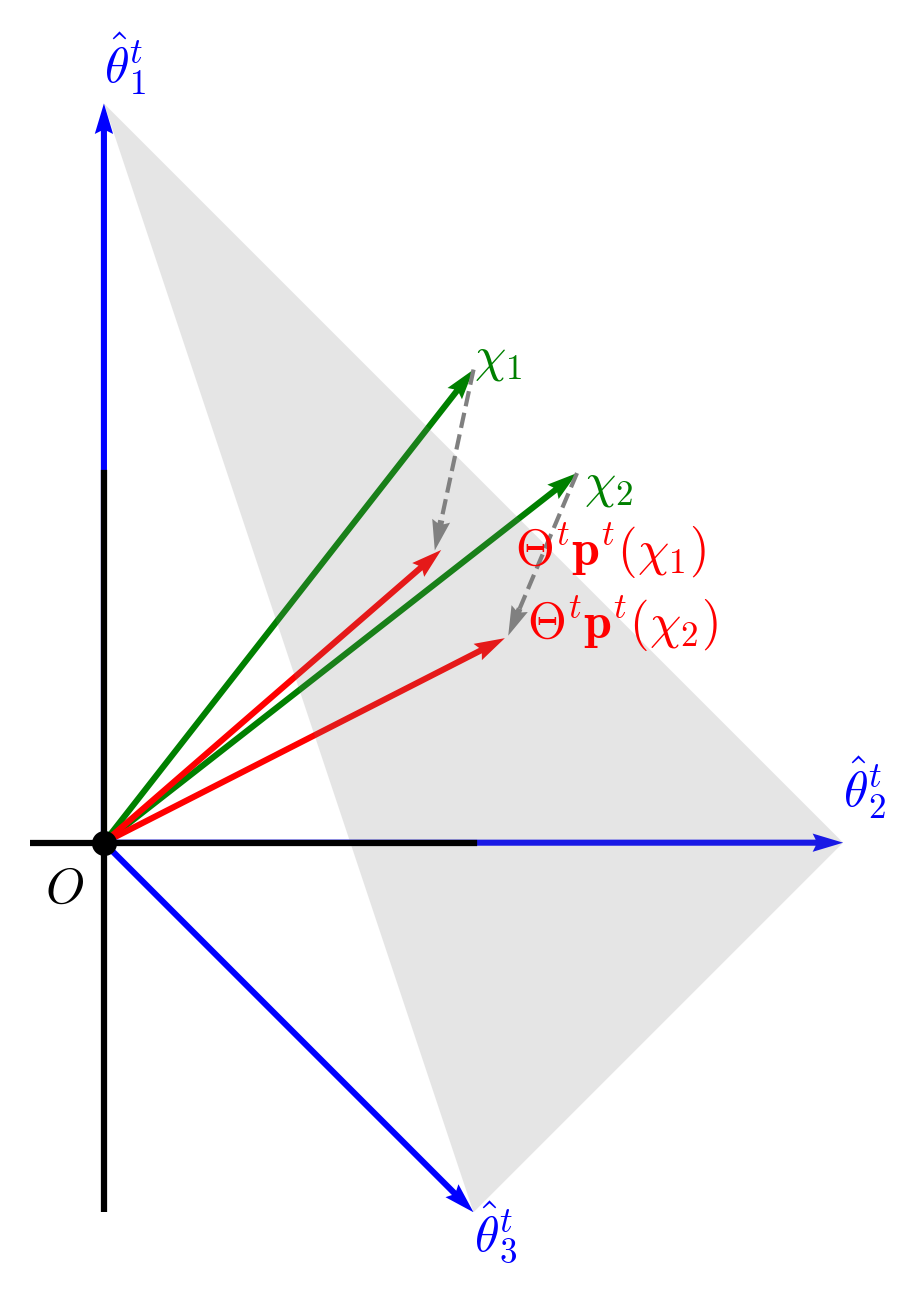}
        \caption{Initial stage for contexts with different optimal arms}
        \label{exam21}
    \end{minipage}
    \hfill
    \begin{minipage}{0.18\textwidth}
        \centering
        \includegraphics[width=\textwidth]{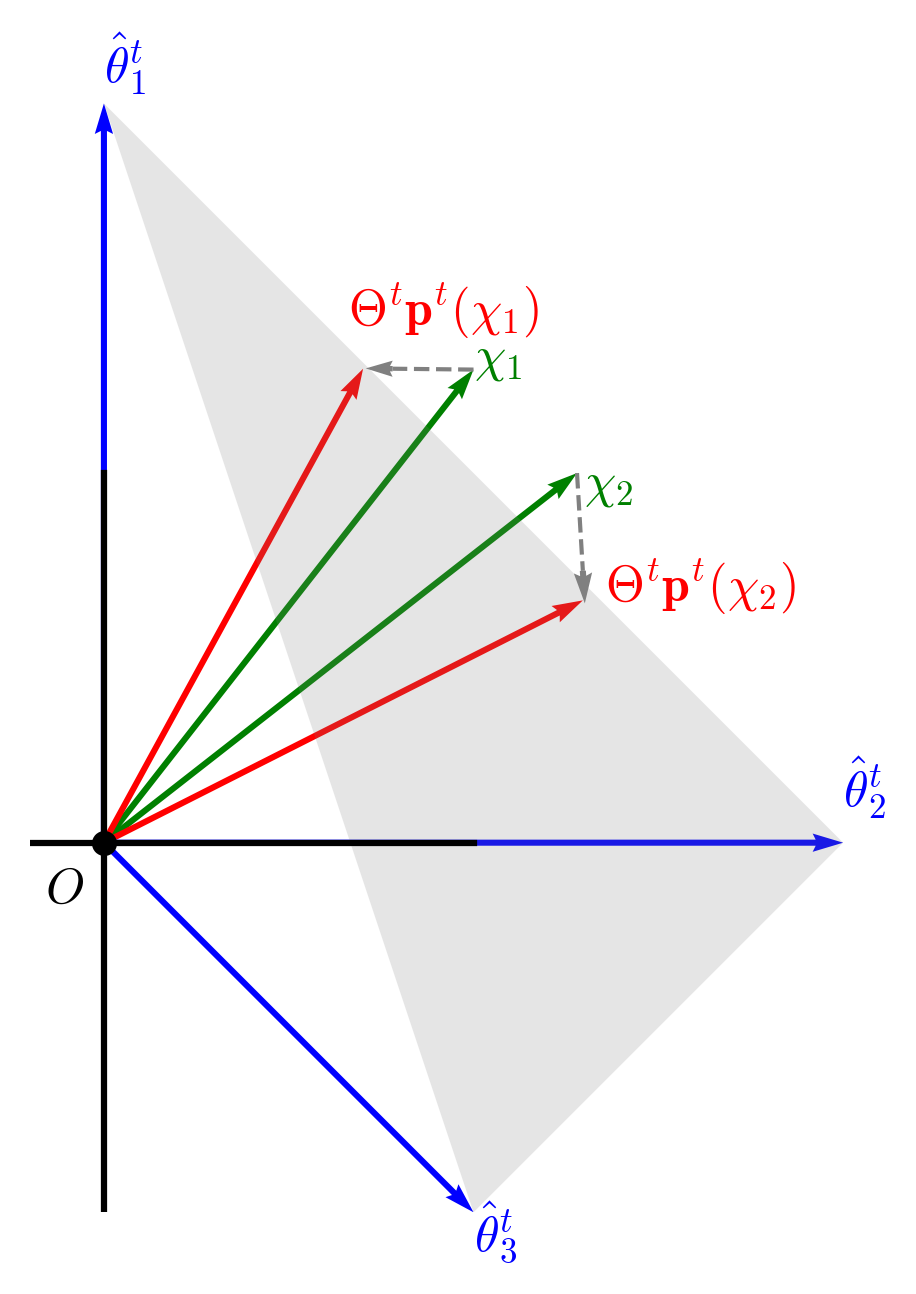}
        \caption{Converging stage for contexts with different optimal arms}
        \label{exam22}
    \end{minipage}
\end{figure}

\begin{theorem}\label{Theorem2}
For the realization of prior, \(\{\theta_k\}_{k \in [K]}\) such that the two contexts \(\chi_1\) and \(\chi_2\) have the different optimal arms \(\alpha_1\) and $\alpha_2$, the frequentist regret of our truthful-TS mechanism in Mechanism \ref{A1} is $O(\ln T)$ to be upper bounded by
\begin{align*}
&\sum_{n=1}^2\bigg( \sum_{j\neq \alpha_1,\alpha_2}\frac{18}{\Delta_{n,j}^2}\ln \frac{T\Delta_{n,j}^2}{36}+C_{n,j}\bigg)\bigg(\frac{\beta_{3-n}}{\beta_{n}}\Delta_{n,\alpha_{3-n}}\\
&+ \bigg(1+\frac{\beta_{3-n}}{\beta_{n}}\bigg)\max_{n=1,2} \Delta_{n,j}\bigg)+\sum_{n=1}^2\bigg(\bigg(2+\frac{\beta_{3-n}}{\beta_{n}}\bigg)\bigg(\frac{18}{\Delta_{n,\alpha_{3-n}}^2}\ln \frac{T\Delta_{n,\alpha_{3-n}}^2}{36}\\
&+C_{n,\alpha_{3-n}}\bigg)+\frac{2\beta_{3-n}}{\beta_{n}}\bigg(\frac{2048}{\Delta_{n}^2}\ln \frac{T\Delta_{n}^2}{2048}+D_n\bigg)\bigg) \Delta_{n,\alpha_{3-n}},
\end{align*}
where \(\Delta_{n,j}=\chi_n^\top \theta_{\alpha_n} - \chi_n^\top \theta_j\) is the reward gap between the optimal arm \(\alpha_n\) and arm \(j \neq \alpha_n\) for context \(\chi_n\). $\Delta_n=\min_{j\neq n}\Delta_{n,j}$ is the minimum reward gap for context $\chi_n$. $C_{n,j}$ is a constant for $n\in\{1,2\}$ and $j\in[K]$, and $D_n$ is the constant part for $n\in\{1,2\}$.
\end{theorem}

\begin{proof}
To upper bound the regret when the contexts have different optimal arms, we need to consider the expected number of times that either context has an incentive to misreport, which introduces an additional part on regret analysis compared to Theorem \ref{Theorem1}. Define \(I(t)\) as the event that the two contexts have conflict at time \(t\), meaning:
\begin{align*}
    I(t): \chi_1^\top\Theta^t(\mathbf{p}^t(\chi_2)-\mathbf{p}^t(\chi_1))>0 \text{ or }\chi_2^\top\Theta^t(\mathbf{p}^t(\chi_1)-\mathbf{p}^t(\chi_2))>0.
\end{align*}
    
    Similar to the proof of Theorem \ref{Theorem1}, we first decompose the regret in (\ref{regret}) as follows:
\begin{align}
        &\mathbb{E}[\mathcal{R}(T)]\nonumber\\
        =&\sum_{t=1}^T\sum_{n=1}^2\sum_{j\neq\alpha_n}\mathbb{E}[\mathbf{1}(x_t=\chi_n,a_t=j)]\Delta_{n,j}\nonumber\\
        \leq&\sum_{t=1}^T\sum_{n=1}^2\bigg(\bigg(\sum_{j\neq\alpha_1,\alpha_2}\mathbb{E}[\mathbf{1}(x_t=\chi_n,\bar{a}_t=j)]\max_{n=1,2}\Delta_{n,j}\bigg)+\Delta_{n,\alpha_{3-n}}\nonumber\\
        &\cdot\mathbb{E}[\mathbf{1}(x_t=\chi_n,a_t=\alpha_{3-n},\overline{I(t)})+\mathbf{1}(x_t=\chi_n,{a}_t=\alpha_{3-n},{I(t)})]\bigg).\label{eq9}
        \end{align}
We now decompose each of the first two terms in the last expression, beginning with the first term, the  $\mathbb{E}\bigl[\mathbf{1}(x_t=\chi_n,\bar{a}_t=j)\bigr]$  
can be written as
\begin{align*}
    &\mathbb{E}[\mathbf{1}(x_t=\chi_n,\bar{a}_t=j)]\\
    =&\mathbb{E}[\mathbf{1}(x_t=\chi_n,\bar{a}_t=j,I(t))+\mathbf{1}(x_t=\chi_n,\bar{a}_t=j,\overline{I(t)})]\\
    =&\mathbb{E}[\mathbf{1}(x_t=\chi_n,\tilde{x}_t=\chi_n,\bar{a}_t=j,I(t))\\
    &+\mathbf{1}(x_t=\chi_n,\tilde{x}_t=\chi_{3-n},\bar{a}_t=j,I(t))\\
    &+\mathbf{1}(x_t=\chi_n,\tilde{x}_t=\chi_n,\bar{a}_t=j,\overline{I(t)})]\\
    \leq&\mathbb{E}[\mathbf{1}(x_t=\chi_n,\tilde{x}_t=\chi_n,\bar{a}_t=j)\\
    &+\frac{\beta_{n}}{\beta_{3-n}}\mathbf{1}(x_t=\chi_{3-n},\tilde{x}_t=\chi_{3-n},\bar{a}_t=j)],
\end{align*}
where the final inequality follows from a derivation analogous to (\ref{eq31}). On the event \(\overline{I(t)}\), the algorithm coincides with standard Thompson sampling, so the expectation \(\mathbb{E}[\mathbf{1}(x_t=\chi_n,\,a_t=\alpha_{3-n},\,\overline{I(t)})]\) in the second term of (\ref{eq9}) can be upper bounded by
\begin{align*}
    &\mathbb{E}[\mathbf{1}(x_t=\chi_n,\tilde{x}_t=\chi_{n},\bar{a}_t=\alpha_{3-n})].
\end{align*}
Inserting the two derived expressions above into the final expression of (\ref{eq9}) yields the following upper bound for (\ref{eq9}):
    \begin{align}
        &\sum_{t=1}^T\sum_{n=1}^2\sum_{j\neq\alpha_1,\alpha_2}\bigg(1+\frac{\beta_{3-n}}{\beta_n}\bigg)\mathbb{E}[\mathbf{1}(x_t=\chi_n,\tilde{x}_t=\chi_n,\bar{a}_t=j)]\max_{n=1,2}\Delta_{n,j}\nonumber\\
        &+(\mathbb{E}[\mathbf{1}(x_t=\chi_n,\tilde{x}_t=\chi_{n},\bar{a}_t=\alpha_{3-n})]\nonumber\\
        &+\mathbb{E}[\mathbf{1}(x_t=\chi_{n},a_t=\alpha_{3-n},{I(t)})])\Delta_{n,\alpha_{3-n}}.\label{eq16}
    \end{align}
        The first two lines of Eq. (\ref{eq16}) can be upper bounded using the procedure for bounding the number of suboptimal pulls described in (\ref{eq17}) to (\ref{eq23}) from the last section. To upper bound the last term, we first present Lemma \ref{conflictupper} that transforms the event including \(I(t)\) to an other event involving \(p_{\alpha_1}^t(\chi_1) + p_{\alpha_2}^t(\chi_2)\).
        

\begin{lemma}\label{conflictupper}
Let \(\underline{k_n}\) denote the empirically worst arm for context \(\chi_n\) in any bandit process at any given time. Let $\epsilon^t 
 $ denote
 \begin{align}
     \min_n\frac{\chi_n^\top\hat{\theta}^t_{\alpha_n}-\chi_n^\top\hat{\theta}^t_{\alpha_{3-n}}}{2(\chi_n^\top\hat{\theta}_{\alpha_n}^t-\chi_n^\top\hat{\theta}^t_{{\underline{k_n}}}+\max_{k\neq \alpha_1,\alpha_2}\chi_n^T(\hat{\theta}_{k}^t-\hat{\theta}_{\underline{k_n}}^t))}.\label{trutherror}
 \end{align}
The probability of event \(I(t)\), where either \(\chi_1\) or \(\chi_2\) has an incentive to misreport at time \(t\), can be upper bounded by the probability of the following event:
\begin{align*}
    E[\mathbf{1}({I(t)})] \leq E\bigg[\mathbf{1}(p^t_{\alpha_n}(\chi_{n})+p^t_{\alpha_{3-n}}(\chi_{3-n})
    \leq 2-2\epsilon^t\bigg].
\end{align*}
\end{lemma}        

The complete proof of this lemma can be found in Appendix \ref{APconflictupper}. Using Lemma \ref{conflictupper}, we then decompose the last term in (\ref{eq16}) below:
\begin{align}
    &\mathbb{E}[\mathbf{1}(x_t=\chi_{n},a_t=\alpha_{3-n},{I(t)})]\nonumber\\
    =&\mathbb{E}[\mathbf{1}(x_t=\chi_{n},\tilde{x}_t=\chi_{3-n},\bar{a}_t=\alpha_{3-n},{I(t)})]\nonumber\\
    &+\mathbb{E}[\mathbf{1}(x_t=\chi_{n},\tilde{x}_t=\chi_{n},\bar{a}_t=\alpha_{3-n},{I(t)})]  \nonumber\\
    \leq&\mathbb{E}[\mathbf{1}(x_t=\chi_{n},\tilde{x}_t=\chi_{3-n},\bar{a}_t=\alpha_{3-n},\nonumber\\
    &{p^t_{\alpha_1}(\chi_1)+p^t_{\alpha_2}(\chi_2)<2-2\epsilon^t})]\nonumber\\
    &+\mathbb{E}[\mathbf{1}(x_t=\chi_{n},\tilde{x}_t=\chi_{n},\bar{a}_t=\alpha_{3-n})]\nonumber\\
    \leq&\mathbb{E}[\mathbf{1}(x_t=\chi_{n},\tilde{x}_t=\chi_{3-n},\bar{a}_t=\alpha_{3-n},{p^t_{\alpha_n}(\chi_n)<1-\epsilon^t})]\nonumber\\
    &+\mathbb{E}[\mathbf{1}(x_t=\chi_{n},\tilde{x}_t=\chi_{3-n},\bar{a}_t=\alpha_{3-n},{p^t_{\alpha_{3-n}}(\chi_{3-n})<1-\epsilon^t})]\nonumber\\
    &+\mathbb{E}[\mathbf{1}(x_t=\chi_{n},\tilde{x}_t=\chi_{n},\bar{a}_t=\alpha_{3-n})]\label{eq33}.
    \end{align}
The first line of the expression above can be upper bounded by
 \begin{align*}   
&\mathbb{E}[\mathbf{1}(x_t=\chi_{n},\tilde{x}_t=\chi_{3-n},{p^t_{\alpha_n}(\chi_n)<1-\epsilon^t})]\\
=&\frac{\beta_{3-n}}{\beta_n}\mathbb{E}[\mathbf{1}(x_t=\chi_{n},\tilde{x}_t=\chi_{n},{p^t_{\alpha_n}(\chi_n)<1-\epsilon^t})]\\
    \leq&\frac{\beta_{3-n}}{\beta_n}\sum_{j\neq \alpha_n}\mathbb{E}[\mathbf{1}(x_t=\chi_{n},\tilde{x}_t=\chi_{n},\bar{a}_t=j)]\\
    &+\frac{\beta_{3-n}}{\beta_n}\mathbb{E}[\mathbf{1}(x_t=\chi_{n},\tilde{x}_t=\chi_{n},\bar{a}_t=\alpha_n,{p^t_{\alpha_n}(\chi_n)<1-\epsilon^t})].
\end{align*}
The first equality exploits that \(p^t_{\alpha_n}(\chi_n)\) is determined by the history up to time \(t\) and is therefore independent of the newly sampled context \(\tilde x_t\). The second line of (\ref{eq33}) equals
\begin{align*}
    &\frac{\beta_{n}}{\beta_{3-n}}\mathbb{E}[\mathbf{1}(x_t=\chi_{3-n},\tilde{x}_t=\chi_{3-n},\bar{a}_t=\alpha_{3-n}\\
    &,{p^t_{\alpha_{3-n}}(\chi_{3-n})<1-\epsilon^t})].
\end{align*}
Substitute them to (\ref{eq33}), we have
\begin{align}
    &\mathbb{E}[\mathbf{1}(x_t=\chi_{n},a_t=\alpha_{3-n},{I(t)})]\nonumber\\
    \leq&\bigg(1+\frac{\beta_{3-n}}{\beta_n}\bigg)\mathbb{E}[\mathbf{1}(x_t=\chi_{n},\tilde{x}_t=\chi_{n},\bar{a}_t=\alpha_{3-n})]\nonumber\\
    &+\frac{\beta_{3-n}}{\beta_n}\sum_{j\neq \alpha_1,\alpha_2}\mathbb{E}[\mathbf{1}(x_t=\chi_{n},\tilde{x}_t=\chi_{n},\bar{a}_t=j)]\nonumber\\
    &+\sum_{n=1}^2\frac{\beta_{3-n}}{\beta_n}\mathbb{E}[\mathbf{1}(x_t=\chi_{n},\tilde{x}_t=\chi_{n},\bar{a}_t=\alpha_n,{p^t_{\alpha_n}(\chi_n)<1-\epsilon^t})]\label{eq34}
\end{align}

    When summing from \( t = 1 \) to \( T \), the first two terms in the right-hand side above can be upper bounded using the same method as in (\ref{eq17}) through (\ref{eq23}) from the last section, then applying Lemmas \ref{lemmaut1}, \ref{lemmaut2}, and \ref{lemmaut3}. The upper bound for the summation from \( t = 1 \) to \( T \) of the last term in (\ref{eq34}) is provided in Lemma \ref{opt_less} below.
   \begin{lemma}\label{opt_less}
Let $\epsilon^t$ denote the expression in (\ref{trutherror}). Then the expected number of pulls of $\alpha_n$ by $\chi_n$ together with the occurrence of event ${p^t_{\alpha_{n}}(\chi_{n})<1-\epsilon^t}$ is upper bounded by:
        \begin{align*}
&\sum_{t=1}^T\mathbb{E}[\mathbf{1}(x_t=\chi_{n},\tilde{x}_t=\chi_{n},\bar{a}_t =\alpha_{n},{p^t_{\alpha_{n}}(\chi_{n})<1-\epsilon^t})]\\
\leq&\frac{2048}{\Delta_{n}^2}\ln \frac{T\Delta_{n}^2}{2048}+D_n,
    \end{align*}
    where $D_n$ is a constant for $n\in\{1,2\}$.
\end{lemma}
   
The complete proof of this lemma can be found in Appendix \ref{APopt_less}. By substituting (\ref{eq34}) into (\ref{eq16}) and then substituting Lemmas \ref{lemmaut1}, \ref{lemmaut2}, \ref{lemmaut3}, and \ref{opt_less}, we derive the final result of Theorem \ref{Theorem2}.
\end{proof}


\section{Simulation Experiments}\label{section7}
In this section, we use simulation experiments to evaluate the performance of our truthful-TS mechanism for more than two contexts and non-Gaussian noise distribution.

We first extend Gaussian noise $\eta_t\sim \mathcal{N}(0,1)$ to Laplace noise as a typical sub-Gaussian distribution in Fig. \ref{experiment2}. Since the Laplace noise yields a non-closed form and non-standard posterior distribution, we turn to numerical methods to update the posterior and compute the probabilities \(\mathbf{p}^t(x)\) in (\ref{TSP}). To ensure a small error of inherent approximation in numerical methods, we conduct this experiment for the small scale case with \(N = d = K = 2\) and $T=50$ time steps. For a fair regret comparison, we set the same Gaussian prior for both noises where prior mean $\mu_1 = (0,1)$ and $\mu_2 = (1,0)$ and prior covariance matrix $V_1=V_2 = I$. The variances of both Laplace and Gaussian noises are set to 1 in Fig. \ref{experiment2}. To compare with our mechanism in Mechanism 1, we use the ideal Thompson sampling (always assuming agents' truthful reporting) as the performance upper bound. Fig. \ref{experiment2} implies that our truthful-TS mechanism yields a similar regret order to Thompson sampling. Since Laplace noise has a heavier tail than Gaussian noise, the regret order under Laplace noise remains sublinear but is slightly higher than that under Gaussian noise. As observed, the regrets for both the truthful-TS mechanism and the Thompson sampling algorithm under Laplace noise will exceed those under Gaussian noise after \( t = 50 \). We can have the same conclusion when extending to other sub-Gaussian noises such as uniform and Cauchy distributions.

Similar to the experiment setting in \cite{bastani2021mostly}, we then extend the setting to $N=3,5,9$ contexts and $d=5,9$ dimensions under Gaussian prior and noise. We consider $K=6$ arms for recommendations among these contexts. For each arm $k$, we set its prior distribution $\mathcal{P}_k\sim\mathcal{N}(\mu_k,V_k)$ such that $\mu_k$ is a vector with a single 1 in the \(k\)-th entry and \(V_k\) is the identity matrix. Different contexts are sampled from a multivariate uniform distribution over \([0,1]^d\). For each parameter setting, we run 100 simulations, generating new realizations of \(\{\theta_k\}_{k\in[K]}\) from the prior distribution of $\mathcal{N}(\mu_k,V_k)$ each time, and calculate the average regret. The results are displayed in Fig. \ref{experiment1}. According to Fig. \ref{experiment1}, like the ideal Thompson sampling, our truthful-TS mechanism still exhibits sublinear order for different $N$ and $d$, which is consistent with our Theorems \ref{Theorem1} and \ref{Theorem2}. Though ideal Thompson sampling algorithm's regret grows with $N$ and $d$, our mechanism grows faster. The reason is that as there are more contexts, a context may envy more other contexts with higher convergence rates and our Mechanism \ref{A1} will reduce the exploitation of those contexts, thereby slowing down the overall convergence.



\begin{figure}
    \centering
        \begin{minipage}{0.2\textwidth}
        \centering
        \includegraphics[width=\textwidth]{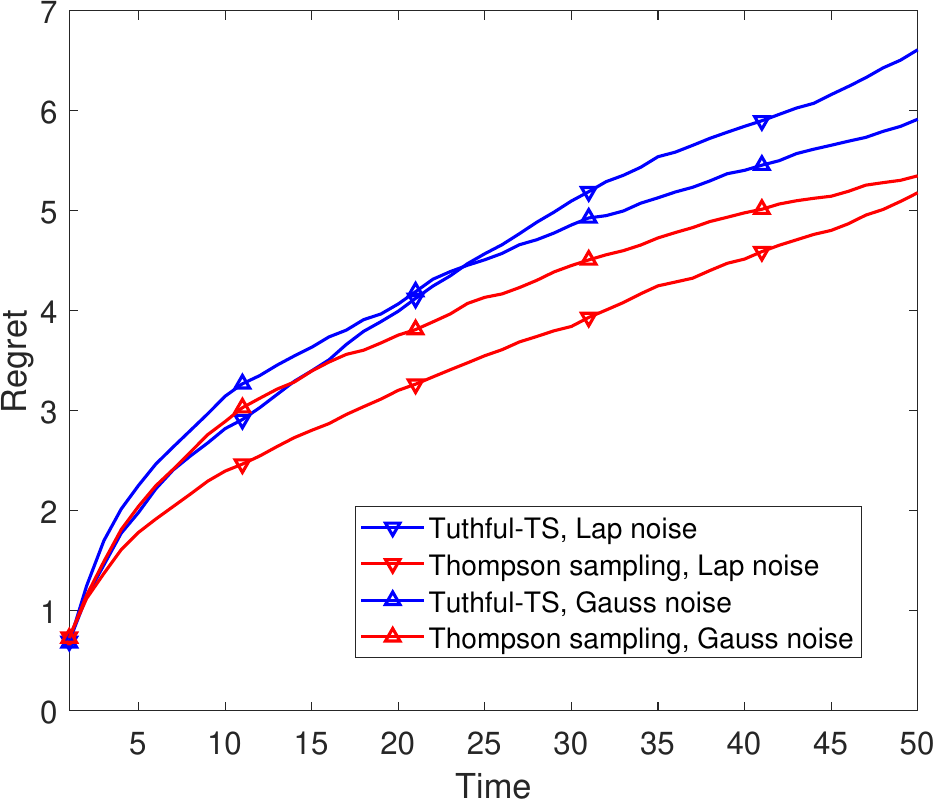}
        \caption{Cumulative regret at each time step under our truthful-TS mechanism in Mechanism \ref{A1} and Thompson sampling algorithm, and Gaussian and Laplace noises with 2 contexts, 2 arms, and 2 dimensions.}
        \label{experiment2}
    \end{minipage}
    \hfill
\begin{minipage}{0.22\textwidth}
    \centering
    \includegraphics[width=\textwidth]{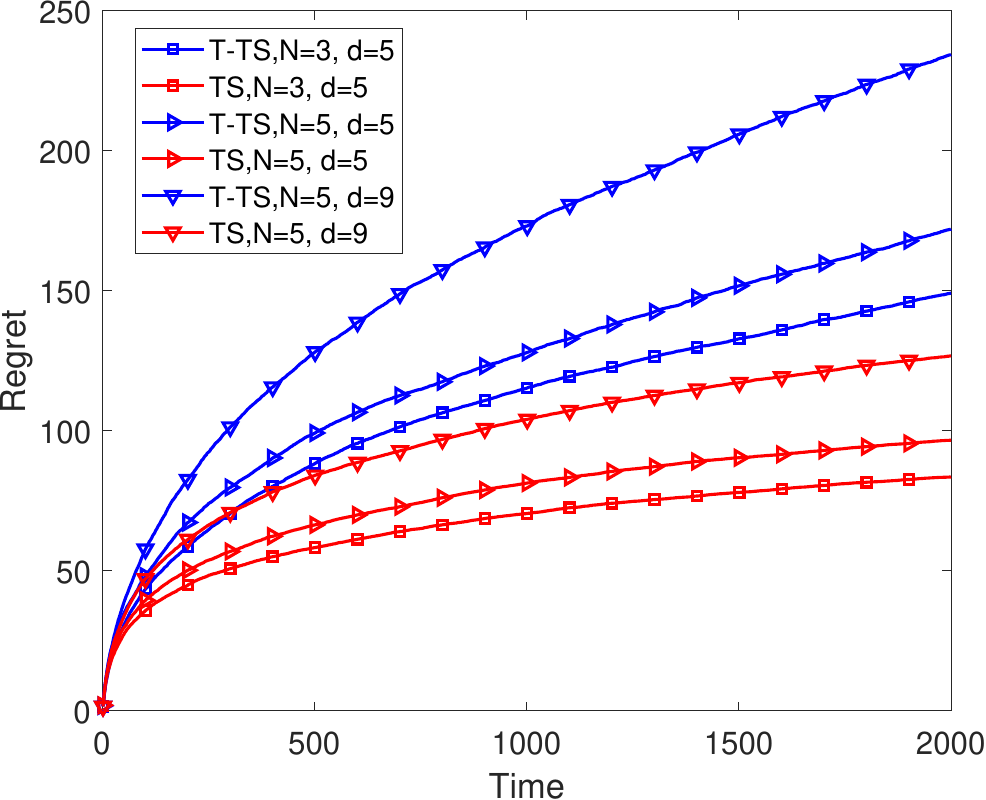}
    \caption{Cumulative regret at each time step under our Mechanism \ref{A1} (T-TS) and Thompson sampling algorithm (TS), where the number of contexts varies between $N\in\{3,5,9\}$ and $d\in\{5,9\}$.}
    \label{experiment1}
    \end{minipage}
\end{figure}


\section{Conclusion}
In this paper, we investigate the problem of strategic misreporting of private contexts by agents within the Bayesian contextual linear bandit framework. We are the first to analyze this issue and demonstrate that existing algorithms fail to perform effectively under such misreporting behavior. To address this, we propose a novel truthful mechanism based on the Thompson sampling algorithm, which solves a LP at each time step to ensure incentive compatibility. We prove that our mechanism achieves a problem-dependent regret bound of \(O(\ln T)\) in the two-context case with Gaussian priors and noise. Furthermore, our numerical results suggest that the proposed mechanism retains a comparable regret order across multiple contexts and under heavier tails of noise. 


\bibliographystyle{ACM-Reference-Format} 
\balance
\bibliography{Arx}

\onecolumn
\appendix
\section{Proof of Lemma \ref{TSL}}\label{APTSL}
Consider a Bayesian contextual linear bandit problem with two contexts: \(\chi_1 = (1,0)\) and \(\chi_2 = (1, \frac{1}{2})\), arriving with probabilities \(\beta_1 = \frac{5}{6}\) and \(\beta_2 = \frac{1}{6}\), respectively. There are two arms drawn from \(\mathcal{N}(\mu_1, V_1)\) and \(\mathcal{N}(\mu_2, V_2)\) with prior means \(\mu_1 = (1,0)\) and \(\mu_2 = (-1,1)\) and prior covariances \(V_1 = V2 = I\). By prior belief, according to the Thompson sampling probabilities in (5), we have \(\mathbf{p}^1(\chi_1) = (0.76, 0.24)\) and \(\mathbf{p}^1(\chi_2) = (0.62, 0.38)\), which means that agent with context $\chi_2$ will have the incentive to misreport as $\chi_1$ to obtain a higher expected reward $\chi_2^\top\Theta^1\mathbf{p}^1(\chi_1)$. Therefore, if the first context \(x_t\) is \( \chi_2 \), this first agent will misreport her context as \(\chi_1\).

We will demonstrate that, under a specific realization of \(\theta_1\) and \(\theta_2\) from the instance above, which satisfy the following conditions:
\begin{align}
    &\frac{1}{2} + \sum_{n=1}^2 \beta_n \chi_n^\top \theta_2 < \sum_{n=1}^2 \beta_n \chi_n^\top \theta_1, \;\chi_2^\top \theta_2 > \chi_2^\top \theta_1 + \epsilon,\; \chi_1^\top \theta_2 < \chi_1^\top \theta_1, \;0 < \sum_{n=1}^2 \beta_n \chi_n^\top \theta_1 < 1 - \delta, \;\delta - 1 < \sum_{n=1}^2 \beta_n \chi_n^\top \theta_2 < 0, \label{realization}
\end{align}
where $0<\epsilon,\;\delta<1$, the Thompson sampling algorithm has a positive probability of failing to correctly identify arm \(\theta_2\) as superior to arm \(\theta_1\) for context \(\chi_2\). Consequently, context \(\chi_2\) will continue misreporting as \(\chi_1\) throughout the learning process, causing the algorithm to converge on selecting arm \(\theta_1\) with probability 1 for both contexts. This outcome will result in linear regret over time. The following proof shows this result by using condition in (\ref{realization}).

To establish this, we first introduce two stochastic processes. Let \(M_k\) denote the total number of times arm \(k\) is pulled over the entire time horizon. Reorder the subset of time steps within \([T]\) during which the Thompson sampling algorithm selects arm \(k\) into a new sequence \(\mathcal{T}_k(T) = \{1, 2, \dots, M_k\}\). For each \(t \in \mathcal{T}_k(T)\), define a stochastic process \(Y^k_t\) as follows:
\begin{align}
    Y^k_t = \left( \chi_2^\top \hat{\theta}^t_k - \chi_2^\top \mu_k + \frac{t-1}{t} \chi_1^\top \mu_k - \sum_{n=1}^2 \beta_n \chi_n^\top \theta_k \right)t, \quad k=1,2.\label{eq1_}
\end{align}
In the following lemma, we prove that \(Y^k_t\) for $k=1,2$ is a martingale under context $\chi_2$'s misreporting.

\begin{lemma}\label{martin}
If context \(\chi_2\) always misreports as \(\chi_1\), the stochastic process \(Y^k_t\)in (\ref{eq1_}) forms a martingale. Moreover, \(Y^k_1 = \mathbb{E}[Y^k_\tau] = -\sum_{n=1}^2 \beta_n \chi_n^\top \theta_k\).
\end{lemma}
\begin{proof}
Assume that at each time step, the misreporting condition: \(\chi_2^\top \Theta^t \mathbf{p}^t(\chi_1) > \chi_2^\top \Theta^t \mathbf{p}^t(\chi_2)\) holds. Under a Gaussian prior and noise, according to lemma \ref{lemmaantimean}, the posterior mean \(\chi_2^\top \hat{\theta}^{t+1}_k\) for \(t \in \mathcal{T}_k(T)\) under misreporting is given by:
\begin{align}
    &\chi_2^\top \hat{\theta}_k^{t+1} = \chi_2^\top \left(I + t\chi_1 \chi_1^\top \right)^{-1} \left(\mu_k + \sum_{\tau \in [t]} r_\tau \chi_1\right) = \chi_2^\top \mu_k - \frac{t}{t+1} \chi_1^\top \mu_k + \frac{\sum_{\tau \in  [t]} r_\tau }{t+1} \nonumber\\
&\Leftrightarrow\chi_2^\top \hat{\theta}_k^{t+1}-\chi_2^\top \mu_k+\frac{t}{t+1} \chi_1^\top \mu_k-\sum_{n=1}^2 \beta_n \chi_n^\top \theta_k= \frac{\sum_{\tau \in  [t]} r_\tau }{t+1}-\sum_{n=1}^2 \beta_n \chi_n^\top \theta_k=\frac{Y^k_{t+1}}{t+1}, \label{mistheta}
\end{align}
where $ [t]=\{1,\dots,t\}$. Then the expectation of $Y^k_t$ is calculated as
\begin{align*}
    \mathbb{E}[Y^k_t] = \mathbb{E}\bigg[\sum_{\tau \in [t]} r_\tau -(t+1)\sum_{n=1}^2 \beta_n \chi_n^\top \theta_k\bigg]=\theta_k^\top\mathbb{E}_{x_t}[\sum_{\tau \in [t]}x_t]-(t+1)\sum_{n=1}^2 \beta_n \chi_n^\top \theta_k=t\theta_k^\top\sum_{n=1}^2 \beta_n \chi_n-(t+1)\sum_{n=1}^2 \beta_n \chi_n^\top \theta_k=-\sum_{n=1}^2 \beta_n \chi_n^\top \theta_k.
\end{align*}
Then the following equalities further show that $Y^k_t$ is a martingale:
\begin{align*}
    \mathbb{E}[Y^k_{t+1} \mid Y^k_t]
    = &\mathbb{E} \left[ \frac{\sum_{\tau \in[t]} r_\tau}{t+1} - \sum_{n=1}^2 \beta_n \chi_n^\top \theta_k \mid Y^k_t \right](t+1) \\
    =& \mathbb{E} \left[ \frac{\sum_{\tau \in[t-1]} r_\tau-t\sum_{n=1}^2 \beta_n \chi_n^\top \theta_k+r_t+t\sum_{n=1}^2 \beta_n \chi_n^\top \theta_k}{t+1} - \sum_{n=1}^2 \beta_n \chi_n^\top \theta_k \mid Y^k_t \right](t+1) \\
    = &\mathbb{E} \left[ \frac{Y^k_t + t \sum_{n=1}^2 \beta_n \chi_n^\top \theta_k + r_t}{t+1} - \sum_{n=1}^2 \beta_n \chi_n^\top \theta_k \mid Y^k_t \right](t+1) \\
    = &Y^k_t + t \sum_{n=1}^2 \beta_n \chi_n^\top \theta_k + \mathbb{E}[r_t] - \sum_{n=1}^2 \beta_n \chi_n^\top \theta_k (t+1) \\
    = &Y^k_t,
\end{align*}
where the last equality follows from $\mathbb{E}[r_t] = \sum_{n=1}^2 \beta_n \chi_n^\top \theta_k$.
Then we obtain Lemma \ref{martin} and complete the proof.
\end{proof} 

Under (\ref{realization}), we have $\mathbb{E}[Y^1_t]<0$ and $\mathbb{E}[Y^2_t]>0$. Then, by Doob's inequality for martingales, we have:
\begin{align*}
    &\mathbb{P} \left( \max_{t \in \mathcal{T}_1(T)} -Y^1_t \geq 1 \right) \leq\mathbb{E}[-Y^1_t]= \sum_{n=1}^2 \beta_n \chi_n^\top \theta_1\Leftrightarrow\mathbb{P} \left( \min_{t \in \mathcal{T}_1(T)} Y^1_t > -1 \right) \geq1-\mathbb{E}[-Y^1_t]= 1-\sum_{n=1}^2 \beta_n \chi_n^\top \theta_1, \\
    &\mathbb{P} \left( \max_{t \in \mathcal{T}_2(T)} Y^2_t \geq 1 \right) \leq \mathbb{E}[Y^2_t]= \sum_{n=1}^2 \beta_n \chi_n^\top \theta_2\Leftrightarrow\mathbb{P} \left( \max_{t \in \mathcal{T}_2(T)} Y^2_t < 1 \right) \geq 1-\mathbb{E}[Y^2_t]= 1-\sum_{n=1}^2 \beta_n \chi_n^\top \theta_2.
\end{align*}
Substituting (\ref{eq1_}) for \(n, k = 1, 2\) into the inequalities above, we obtain:
\begin{align*}
    \mathbb{P} \left( \min_{t \in \mathcal{T}_1(T)} Y^1_t > -1 \right) 
    = &\mathbb{P} \left( \min_{t \in \mathcal{T}_1(T)} \left( \chi_2^\top \hat{\theta}^t_1 - \chi_2^\top \mu_1 + \frac{t-1}{t} \chi_1^\top \mu_1 - \sum_{n=1}^2 \beta_n \chi_n^\top \theta_1 \right)t > -1 \right) \\
    = &\mathbb{P} \left( \min_{t \in \mathcal{T}_1(T)} \chi_2^\top \hat{\theta}^t_1 > \sum_{n=1}^2 \beta_n \chi_n^\top \theta_1 \right) > 1 - \sum_{n=1}^2 \beta_n \chi_n^\top \theta_1,
\end{align*}
and
\begin{align*}
    \mathbb{P} \left( \max_{t \in \mathcal{T}_2(T)} Y^2_t < 1 \right) 
    = &\mathbb{P} \left( \max_{t \in \mathcal{T}_2(T)} \left( \chi_2^\top \hat{\theta}^t_2 - \chi_2^\top \mu_2 + \frac{t-1}{t} \chi_1^\top \mu_2 - \sum_{n=1}^2 \beta_n \chi_n^\top \theta_2 \right)t < 1 \right) \\
    = &\mathbb{P} \left( \max_{t \in \mathcal{T}_1(T)} \chi_2^\top \hat{\theta}^t_2 < \frac{1}{2} + \sum_{n=1}^2 \beta_n \chi_n^\top \theta_2 \right) > 1 - \sum_{n=1}^2 \beta_n \chi_n^\top \theta_2.
\end{align*}
Recall that under the realization of \(\theta_1\) and \(\theta_2\) in (\ref{realization}), we have \(\sum_{n=1}^2 \beta_n \chi_n^\top \theta_1 > \frac{1}{2} + \sum_{n=1}^2 \beta_n \chi_n^\top \theta_2\). By combining this with the two inequalities above, we have:
\begin{align*}
   \mathbb{P}\bigg(\min_{t \in \mathcal{T}_1(T)} \chi_2^\top \hat{\theta}^t_1 > \max_{t \in \mathcal{T}_1(T)} \chi_2^\top \hat{\theta}^t_2\bigg)>\bigg(1 - \sum_{n=1}^2 \beta_n \chi_n^\top \theta_2\bigg)\bigg(1 - \sum_{n=1}^2 \beta_n \chi_n^\top \theta_2\bigg)>\delta^2.
\end{align*}
This result implies that, starting from time \(t = 1\), where context \(\chi_2\) begins to misreport, there is a probability at least \(\delta^2\) that the Thompson sampling algorithm will misidentify arm \(\theta_1\) as the optimal arm for context \(\chi_2\). Simultaneously, \(\chi_2\) will persist in misreporting at each subsequent time \(t\), upon observing \(\hat{\theta}_k^t\) for \(k = 1, 2\). As a result, its posterior estimate \(\chi_2^\top \hat{\theta}_k^t\) will follow the martingale \(Y_k^t\), forming a logical loop with the event \(\min_{t \in \mathcal{T}_1(T)} \chi_2^\top \hat{\theta}^t_1 > \max_{t \in \mathcal{T}_1(T)} \chi_2^\top \hat{\theta}^t_2\) to happen. Additionally, when \(\chi_2\) misidentifies arm 1 as the optimal arm, context \(\chi_1\) must correctly identify arm 1 as the optimal arm by calculating \(\chi_1^\top \hat{\theta}^t_k\) similarly to Eq. (\ref{mistheta}), because:
\begin{align*}
    \chi_2^\top \hat{\theta}_1^{t+1} &> \chi_2^\top \hat{\theta}_2^{t+1} \\
    1 - \frac{t}{t+1}  + \frac{\sum_{\tau \in \mathcal{T}_{1,1}(t+1)} r_\tau + \sum_{\tau \in \mathcal{T}_{2,1}(t+1)} r_\tau}{t+1}&> -\frac{1}{2} + \frac{t}{t+1}  + \frac{\sum_{\tau \in \mathcal{T}_{1,2}(t+1)} r_\tau + \sum_{\tau \in \mathcal{T}_{2,2}(t+1)} r_\tau}{t+1}\Rightarrow\\
    1 - \frac{t}{t+1}  + \frac{\sum_{\tau \in \mathcal{T}_{1,1}(t+1)} r_\tau + \sum_{\tau \in \mathcal{T}_{2,1}(t+1)} r_\tau}{t+1}&> -1 + \frac{t}{t+1}  + \frac{\sum_{\tau \in \mathcal{T}_{1,2}(t+1)} r_\tau + \sum_{\tau \in \mathcal{T}_{2,2}(t+1)} r_\tau}{t+1},
\end{align*}
where the two sides in the last line are $\chi_1^\top \hat{\theta}^t_1$ and $\chi_1^\top \hat{\theta}^t_2$, respectively. Since arm 1 is always the posterior optimal for both contexts, the probability of choosing arm 1 will always be higher than \(1/2\). The regret, therefore, is at least
\begin{align*}
    \frac{1}{2}\delta^2\beta_2T(\chi_2^\top\theta_2-\chi_2^\top\theta_1)\geq\beta_2 \frac{1}{2} \delta^2 \epsilon T
\end{align*}
under the condition in Eq. (\ref{realization}). Furthermore, the region of \(\theta_1\) and \(\theta_2\) in Eq. (\ref{realization}) has a non-zero measure, implying that this realization has a positive probability of occurring. This completes the proof.

\section{Proof of Lemma \ref{feasible}}\label{APfeasible}

Based on the second constraint of LP (\ref{LP1}), any feasible solution $\{\mathbf{q}^t(\chi_n)\}_{n\in[N]}$ must take the form \(\sum_{i\in [N]}\beta_i\mathbf{p}^t(\chi_i) + \xi_n, \forall n\in[N]\), where \(\xi_n = (\xi_{n,1}, \xi_{n,2}, \dots, \xi_{n,K})\) may be positive or negative, and \(\sum_{n=1}^N\beta_n\xi_n = 0\). We can thus rewrite problem (\ref{LP1}) as follows:
\begin{align}
    \min \max_{n\in[N]}&\left\| \sum_{i\in [N]}\beta_i\mathbf{p}^t(\chi_i)-\mathbf{p}^t(\chi_n)  + \xi_n\right\|_{\infty} \nonumber \\
    \text{s.t.} &\quad \chi_i^\top \Theta^t (\xi_i - \xi_j) \geq 0, \quad \forall i\neq j, \ i,j\in[N] \nonumber \\
    &\quad \xi_{n,1} + \dots + \xi_{n,K} = 0, \quad \forall n\in[N] \nonumber \\
    &\quad \beta_1\xi_{1,k} + \dots + \beta_N\xi_{N,k} = 0, \quad \forall k\in[K] \nonumber \\
    &\quad 0 \leq {\sum_{n\in[N]}\beta_n p_k^t(\chi_n)} + \xi_{n,k} \leq 1, \quad \forall n\in[N], k\in[K]. \label{LP2}
\end{align}

Using the second and third equality constraints, we can set \(\xi_{n,K} = -\xi_{n,1} - \dots - \xi_{n,K-1} \ \forall, n \in [N-1]\) and \(\xi_{N,k} = -\frac{\beta_1}{\beta_N}\xi_{1,k} - \dots - \frac{\beta_{N-1}}{\beta_N}\xi_{N-1,k}, \ \forall k \in [K-1]\). We will then use this to reformulate the first three constraints in (\ref{LP2}) using only variables of \(\xi_{i,1}, \dots, \xi_{i,K-1}\) for \(i = 1, \dots, N-1\). Define $\Theta^t_{1:K-1}$ as $\Theta^t(\cdot,1:K-1)$, which is the submatrix of the first $K-1$ columns in $\Theta^t$. Substituting this $\xi_{n,K}$ and $\xi_{N,k}$ into the first constraint, we get:
\begin{align}
    \chi_i^\top \Theta^t (\xi_i - \xi_j) &= \chi_i^\top \Theta^t \begin{pmatrix}
        \xi_{i,1} - \xi_{j,1} \\
        \vdots \\
        \xi_{i,K-1} - \xi_{j,K-1} \\
        -(\xi_{i,1} - \xi_{j,1}) - \dots - (\xi_{i,K-1} - \xi_{j,K-1})
    \end{pmatrix} = \chi_i^\top (\Theta^t_{1:K-1} - \hat{\theta}^t_K\mathbf{1}^\top) \begin{pmatrix}
        \xi_{i,1} - \xi_{j,1} \\
        \vdots \\
        \xi_{i,K-1} - \xi_{j,K-1}
    \end{pmatrix}.\label{ij}
\end{align}
The above equation holds for \(i \neq N\) and \(j \neq N\). When \(i \neq N\) and \(j = N\), the constraint should be expressed as:
\begin{align}
    &\chi_i^\top \Theta^t (\xi_i - \xi_N)=\chi_i^\top (\Theta^t_{1:K-1} - \hat{\theta}^t_K\mathbf{1}^\top) \begin{pmatrix}
        (1+\frac{\beta_i}{\beta_N})\xi_{i,1} + \sum_{j\neq i,N}\frac{\beta_j}{\beta_N}\xi_{j,1} \\
        \vdots \\
        (1+\frac{\beta_i}{\beta_N})\xi_{i,K-1} + \sum_{j\neq i,N}\frac{\beta_j}{\beta_N}\xi_{j,K-1}
    \end{pmatrix}.\label{iN}
\end{align}
When \(i = N\) and \(j \neq N\), the constraint should be expressed as:
\begin{align*}
    &\chi_N^\top \Theta^t (\xi_N - \xi_j)
    = \chi_N^\top (\Theta^t_{1:K-1} - \hat{\theta}^t_K\mathbf{1}^\top) \begin{pmatrix}
        -(1+\frac{\beta_j}{\beta_N})\xi_{j,1} - \sum_{i\neq j,N}\frac{\beta_i}{\beta_N}\xi_{i,1} \\
        \vdots \\
        -(1+\frac{\beta_j}{\beta_N})\xi_{j,K-1} - \sum_{i\neq j,N}\frac{\beta_i}{\beta_N}\xi_{i,K-1}
    \end{pmatrix}.
\end{align*}

Let \(\mathcal{E}\) denote the matrix formed by the first \(K-1\) rows and \(N-1\) columns of matrix \((\xi_1, \dots, \xi_N)\). Combining Eqs. (\ref{ij}) and (\ref{iN}), the constraint for truthful reporting of \(\chi_i\) with \(i \neq N\), \(\chi_i^\top \Theta^t (\xi_i - \xi_j) \geq 0\) for any \(j \neq i\), can be rewritten as:
\begin{align}
    &\text{vec}((\chi_i^\top (\Theta^t_{1:K-1} - \hat{\theta}^t_K\mathbf{1}^\top)\mathcal{E}A_i)^\top)
    = (\chi_i^\top (\Theta^t_{1:K-1} - \hat{\theta}^t_K\mathbf{1}^\top) \otimes A_i^\top )\text{vec}(\mathcal{E}^\top) \geq 0, \text{ where}\label{Ai}\\
    &A_i = \begin{pmatrix}
        -1 & 0 & \dots & 0 & \frac{\beta_1}{\beta_N} \\
        0 & -1 & \dots & 0 & \frac{\beta_2}{\beta_N} \\
        \vdots & \vdots & \dots & \vdots & \vdots \\
        1 & 1 & \dots & 1 & 1+\frac{\beta_i}{\beta_N} \\
        \vdots & \vdots & \dots & \vdots & \vdots \\
        0 & 0 & \dots & -1 & \frac{\beta_{N-1}}{\beta_N}
    \end{pmatrix} \nonumber
\end{align}
and the row \([1, \dots, 1, 2]\) appears in the \(i\)-th row of the matrix on the left-hand side.

Similarly, the constraint for truthful reporting with \(\chi_N\) can be expressed as:
\begin{align}
    &\text{vec}((\chi_N^\top (\Theta^t_{1:K-1} - \hat{\theta}^t_K\mathbf{1}^\top)\mathcal{E}A_N)^\top) 
    = (\chi_N^\top (\Theta^t_{1:K-1} - \hat{\theta}^t_K\mathbf{1}^\top) \otimes A_N^\top) \text{vec}(\mathcal{E}^\top) \geq 0,\text{ where}\label{AN}\\
    &A_N = \begin{pmatrix}
        -1-\frac{\beta_1}{\beta_N} & -\frac{\beta_1}{\beta_N} & \dots & -\frac{\beta_{1}}{\beta_N} & -\frac{\beta_{1}}{\beta_N} \\
        -\frac{\beta_2}{\beta_N} & -1-\frac{\beta_2}{\beta_N} & \dots & -\frac{\beta_{2}}{\beta_N} & -\frac{\beta_{2}}{\beta_N} \\
        \vdots & \vdots & \dots & \vdots & \vdots \\
        -\frac{\beta_{N-1}}{\beta_N} & -\frac{\beta_{N-1}}{\beta_N} & \dots & -\frac{\beta_{N-1}}{\beta_N} & -1-\frac{\beta_{N-1}}{\beta_N}
    \end{pmatrix} \nonumber.
\end{align}

Stacking the inequalities in Eqs. (\ref{Ai}) and (\ref{AN}) for all \(i \in [N]\) forms the following convex cone:
\begin{align}
\begin{pmatrix}
    \chi_1^\top (\Theta^t_{1:K-1} - \hat{\theta}^t_K\mathbf{1}^\top) \otimes A_1^\top \\
    \vdots \\
    \chi_N^\top (\Theta^t_{1:K-1} - \hat{\theta}^t_K\mathbf{1}^\top) \otimes A_N^\top
\end{pmatrix}
\text{vec}(\mathcal{E}^\top) \geq 0,\label{halfspace}
\end{align}
where the convex cone has a non-zero measure when there is no group of rows in the matrix of the left-hand side pointing in the opposite direction. Formally, let the rows of this matrix be denoted as \(b_1, \dots, b_{N(N-1)}\), and define \(C = [N(N-1)]\). If the following condition holds for all nonnegative \(\lambda_n\) (not all zero):
\[
-\sum_{n \in D} \lambda_n b_n \neq \sum_{n \in C \setminus D} \lambda_n b_n, \quad \forall D \subset C,
\]
then the set of inequalities in (\ref{halfspace}) will have a non-zero measure. This condition is equivalent to stating that the origin 0 does not lie within the convex hull of \(b_1, \dots, b_{N(N-1)}\). A sufficient condition for the origin not to lie within the convex hull of \(b_1, \dots, b_{N(N-1)}\) is that the vectors \(b_1, \dots, b_{N(N-1)}\) are linearly independent. When the vectors are linearly independent, the origin cannot be expressed as a convex combination of them, thereby ensuring that the condition is satisfied. Then we further give a sufficient condition for \(b_1, \dots, b_{N(N-1)}\) being linearly independent in the following lemma.

\begin{lemma}\label{kron}
    If $ \chi_1^\top (\Theta^t_{1:K-1} - \hat{\theta}^t_K\mathbf{1}^\top),\dots, \chi_N^\top (\Theta^t_{1:K-1} - \hat{\theta}^t_K\mathbf{1}^\top)$ are linearly independent, all the rows in (\ref{halfspace}) must also be linearly independent.
\end{lemma}
\begin{proof}
We prove the lemma by contradiction. For simplification, let \( X_n = \chi_n^\top (\Theta^t_{1:K-1} - \hat{\theta}^t_K \mathbf{1}^\top) \). Define \( A_n^\top = (a_{n,1}, \dots, a_{n,N-1})^\top \), where \( a_{n,i} \) represents the \(i\)-th column of \( A_n \). It follows directly from (\ref{Ai}) and (\ref{AN}) that \( a_{n,1}, \dots, a_{n,N-1} \) must be linearly independent.

Assume that the vectors \( X_n \) for all \( n \in [N] \) are linearly independent. Suppose there exists a set of scalars \( \lambda_{i,j} \) for \( i \in [N], j \in [N-1] \), not all zero, such that:
\[
\sum_{i \in [N]} \sum_{j \in [N-1]} \lambda_{i,j} X_i \otimes a_{i,j}^\top = 0.
\]
We can rewrite this expression as:
\[
\sum_{i \in [N]} \sum_{j \in [N-1]} \lambda_{i,j} X_i \otimes a_{i,j}^\top = \sum_{i \in [N]} X_i \otimes \left( \sum_{j \in [N-1]} \lambda_{i,j} a_{i,j}^\top \right).
\]

Since \( X_1, \dots, X_N \) are linearly independent, the only way for this equation to hold is:
\[
\sum_{j \in [N-1]} \lambda_{i,j} a_{i,j}^\top = 0 \quad \text{for each } i \in [N].
\]
However, because \( a_{i,1}, \dots, a_{i,N-1} \) are linearly independent, it follows that all \( \lambda_{i,j} = 0 \). This contradicts our initial assumption that not all \( \lambda_{i,j} \) are zero. Therefore, we reach a contradiction, completing the proof.

\end{proof}
Given the conclusion that the first three constraints of (\ref{LP2}) define a non-zero measure space under the condition in Lemma \ref{kron}, we now turn our attention to the fourth constraint in problem (\ref{LP2}). In the Thompson sampling algorithm, the probability of arm choice, \(\mathbf{p}^t(\chi_n)\), ensures that each arm has a strictly positive probability of being selected at every time step. This condition guarantees that the point \(\sum_{n \in [N]} \beta_n p_k^t(\chi_n)\) lies within the interior of the simplex \(\Delta^K\), formed by all possible distributions over \(K\) arms.

As a result, there exists a positive \(\delta\) such that a \(\delta\)-ball centered at \(\sum_{n \in [N]} \beta_n p_k^t(\chi_n)\) is entirely contained within the simplex. This implies when the condition in Lemma 4.1 holds, the intersection between a \(\delta\)-ball around the origin and the convex cone defined by (\ref{halfspace}) constitutes a feasible space for problem (\ref{LP2}) with non-zero measure. Then we complete the proof.


\section{Proof of Lemma \ref{lemmaut1}}\label{APlemmaut1}
We first present the concentration and anti-concentration inequalities for posterior estimate $\chi_n^\top \hat{\theta}^t_k$ and Thompson sampling value $\tilde{\theta}^t_{n,k}$ by context $\chi_n$ for all $k\in[K]$ in Lemmas \ref{lemmaantimean} and \ref{fact1} used throughout the remaining proofs.

\begin{lemma}\label{fact1}
The following concentration and anti-concentration inequalities are from [3].
If $x\geq 0$ and $Z\sim \mathcal{N}(\mu,\sigma^2)$,
    \begin{align*}
       \sqrt{\frac{2}{\pi}}\frac{e^{-x^2/2}}{x+\sqrt{x^2+8/\pi}} \geq Pr(Z>\mu+x\sigma)\geq \sqrt{\frac{2}{\pi}}\frac{e^{-x^2/2}}{(x+\sqrt{x^2+4})}
    \end{align*}
\end{lemma}

\begin{lemma}\label{lemmaantimean}
For any arm \(i \in [K]\), when an algorithm collects \(t-1\) samples which are all from context \(\chi_n\) at time \(t\), the posterior estimate \(\chi_n^\top \hat{\theta}_i^t\) follows a Gaussian distribution:
\begin{align}
\mathcal{N}\left(\chi_n^\top \theta_i + \frac{\chi_n^\top(\mu_i - \theta_i)}{\chi_n^\top V_i \chi_n} \|\chi_n\|^2_{\hat{V}_i^t}, \|\chi_n\|^2_{\hat{\Sigma}_i^t} \right),
\end{align}
where \( \hat{\Sigma}_i^t = (t-1) \hat{V}_{i}^t \chi_n \chi_n ^\top\hat{V}_{i}^t \), and \( \|\chi_n\|^2_{\hat{V}_i^t} = \chi_n^\top \hat{V}_{i}^t \chi_n \). The matrix \( \hat{V}_{i}^t \) has the closed-form expression:
\begin{align}
\|\chi_n\|^2_{\hat{V}_i^t} = \frac{\chi_n^\top \chi_n}{1 + (t-1) \chi_n^\top \chi_n}.
\end{align}
Furthermore, for any $\delta>0$, the posterior estimate \( \chi_n^\top \hat{\theta}_i^{t} \) satisfies the following concentration inequality:
\begin{align}
&\mathbb{P}\left( \left| \chi_n^\top \theta_i - \chi_n^\top \hat{\theta}_i^{t} \right| \leq \left( \left|\frac{\chi_n^\top (\mu_i - \theta_i)}{\chi_n^\top V_i \chi_n} \right|\|\chi_n\|_{\hat{V}_i^t} + \delta \right) \|\chi_n\|_{\hat{V}_{i}^t} \right)
\geq 1 - \frac{2 \exp\left(-\delta^2 / 2\right)}{\sqrt{\pi} \left( \delta/\sqrt{2} + \sqrt{\delta^2/2 + 4/\pi} \right)}.\label{eq7}
\end{align}
\end{lemma}
\begin{proof}

We begin by writing the difference between the posterior estimate and the true parameter \(\chi_n^\top \theta_i \) as follows:
\begin{align*}
    \chi_n^\top(\hat{\theta}_{i}^t - \theta_i) = \chi_n^\top \left( \hat{V}_{i}^t \left( V_i^{-1} (\mu_i - \theta_i) + \sum_{\tau \in \mathcal{T}_i^t} x_\tau \eta_\tau \right) \right),
\end{align*}
where \( \mathcal{T}_i^t \) represents the set of time steps when the algorithm chose arm \( i \) before time $t$. Therefore, the posterior mean of \( \chi_n^\top \hat{\theta}_{i}^t \) follows a Gaussian distribution with mean \( \chi_n^\top \theta_i + \chi_n^\top \hat{V}_i^t V_i^{-1} (\mu_i - \theta_i) \) and variance \( \|\chi_n\|_{\hat{\Sigma}_i^t}^2 = \chi_n^\top \hat{V}_{i}^t \left( \sum_{\tau \in \mathcal{T}_i} x_{\tau} x_\tau^\top \right) \hat{V}_{i}^t \chi_n < \|\chi_n\|_{\hat{V}_i^t}^2 \).

When all the collected \( t-1 \) samples are from the context \( \chi_n \), we apply the Sherman–Morrison–Woodbury formula to obtain:
\begin{align*}
    \chi_n^\top \hat{V}^t_k \chi_n =& \chi_n^\top \left( V_i^{-1} + (t-1) \chi_n \chi_n^\top \right)^{-1} \chi_n
    =\chi_n^\top V_i \chi_n - \frac{(t-1) (\chi_n^\top V_i \chi_n)^2}{1 + s \chi_n^\top V_i \chi_n}
    = \frac{\chi_n^\top V_i \chi_n}{1 + (t-1) \chi_n^\top V_i \chi_n}.
\end{align*}
Also by Sherman–Morrison–Woodbury formula, we further expand the term \( \chi_n^\top \hat{V}_i^t V_i^{-1} (\mu_i - \theta_i) \):
\begin{align*}
    \chi_n^\top \hat{V}_i^t V_i^{-1} (\mu_i - \theta_i)
    =& \chi_n^\top \left( V_i - \frac{(t-1) V_i \chi_n \chi_n^\top V_i}{1 + (t-1) \chi_n^\top V_i \chi_n} \right) V_i^{-1} (\mu_i - \theta_i)\\
    =& \frac{\chi_n^\top (\mu_i - \theta_i)}{1 + (t-1) \chi_n^\top V_i \chi_n}
    = \frac{\chi_n^\top (\mu_i - \theta_i)}{\chi_n^\top V_i \chi_n} \cdot \frac{\chi_n^\top V_i \chi_n}{1 + (t-1) \chi_n^\top V_i \chi_n}
    = \frac{\chi_n^\top (\mu_i - \theta_i)}{\chi_n^\top V_i \chi_n} \|\chi_n\|^2_{\hat{V}_i^t}.
\end{align*}

Finally, applying the concentration inequality for the Normal distribution (Lemma \ref{fact1}), we get:
\begin{align*}
    \mathbb{P}\left( \left| \chi_n^\top \theta_i - \chi_n^\top \hat{\theta}_i^{t} \right| \leq \left( \left|\frac{\chi_n^\top (\mu_i - \theta_i)}{\chi_n^\top V_i \chi_n} \right|\|\chi_n\|_{\hat{V}_i^t} + \delta \right) \|\chi_n\|_{\hat{V}_i^t} \right)
    \geq &\mathbb{P}\left( \left| \chi_n^\top \hat{\theta}_i^{t} - \chi_n^\top \theta_i - \frac{\chi_n^\top (\mu_i - \theta_i)}{\chi_n^\top V_i \chi_n} \|\chi_n\|^2_{\hat{V}_i^t} \right| \leq \delta \|\chi_n\|_{\hat{\Sigma}_i^t} \right)\\
    \geq& 1 - \frac{2 \exp\left(-\delta^2 / 2\right)}{\sqrt{\pi} \left( \delta / \sqrt{2} + \sqrt{\delta^2 / 2 + 4 / \pi} \right)}.
\end{align*}

\end{proof}

To prove Lemma \ref{lemmaut1}, we present another key lemma below as in \cite{agrawal2017near}.
\begin{lemma}\label{lemmaut11}
In Thompson sampling, for any suboptimal arm \( j \neq \alpha \) and any history \( \mathcal{F}_t \), the probability of selecting arm \( j \) at time \( t \), together with the occurrence of events \( E^{\mu}_{n,j}(t) \) and \( E^{\theta}_{n,j}(t) \), can be upper-bounded as follows:
\begin{align}
    &\mathbb{P}(\bar{a}_t = j, E^{\mu}_{n,j}(t), E^{\theta}_{n_j}(t) \mid \mathcal{F}_t, \tilde{x}_t = \chi_n) 
    \leq \frac{\mathbb{P}(\tilde{\theta}_{n,\alpha} \leq \chi_n^\top \theta_j + \frac{2\Delta_{n,j}}{3}\mid\mathcal{F}_t, \tilde{x}_t = \chi_n)}{\mathbb{P}(\tilde{\theta}_{n,\alpha} >\chi_n^\top \theta_j + \frac{2\Delta_{n,j}}{3}\mid\mathcal{F}_t, \tilde{x}_t = \chi_n)} \cdot \mathbb{P}(\bar{a}_t = \alpha \mid E^{\mu}_{n,j}(t), E^{\theta}_{n_j}(t),\mathcal{F}_t, \tilde{x}_t = \chi_n). \label{eq2}
\end{align}
\end{lemma}

\begin{proof}
When the Thompson sampling algorithm chooses arm \( j \) conditioned on events \( E^{\mu}_{n,j}(t) \) and \( E^{\theta}_{n,j}(t) \), since the algorithm bases this choice on the sampled value \( \tilde{\theta}_{n,j}^t \), it must follow that the sampled values for all other arms are less than or equal to \( \chi_n^\top \theta_j + \frac{2\Delta_{n,j}}{3} \). Therefore, we have:
\begin{align*}
\mathbb{P}(\bar{a}_t = j, E^{\mu}_{n,j}(t), E^{\theta}_{n_j}(t) \mid \mathcal{F}_t, \tilde{x}_t = \chi_n)
    \leq&\mathbb{P}(\bar{a}_t = j \mid E^{\mu}_{n_j}(t), E^{\theta}_{n_j}(t), \mathcal{F}_t, \tilde{x}_t = \chi_n)\\
    \leq &\mathbb{P}(\tilde{\theta}_{n,i} \leq \chi_n^\top \theta_j + \frac{2\Delta_{n,j}}{3}, \forall i \neq j\mid \mathcal{F}_t, \tilde{x}_t = \chi_n)\\
=&\mathbb{P}(\tilde{\theta}_{n,\alpha} \leq \chi_n^\top \theta_j + \frac{2\Delta_{n,j}}{3}\mid\mathcal{F}_t, \tilde{x}_t = \chi_n) \cdot \mathbb{P}(\tilde{\theta}_{n,i} \leq \chi_n^\top \theta_j + \frac{2\Delta_{n_j}}{3}, \forall i \neq j, \alpha\mid \mathcal{F}_t, \tilde{x}_t = \chi_n).
\end{align*}

On the other hand, if the sampling value for the optimal arm \( \tilde{\theta}_{n,\alpha} \) is higher than \( \chi_n^\top \theta_j + \frac{2\Delta_{n,j}}{3} \), the algorithm must select arm \( \alpha \). Therefore, we have:
\begin{align*}
&\mathbb{P}(\bar{a}_t=\alpha \mid E^{\mu}_{n_j}(t), E^{\theta}_{n_j} (t),\mathcal{F}_t, \tilde{x}_t = \chi_n)
\geq \mathbb{P}(\tilde{\theta}_{n,\alpha} > \chi_n^\top \theta_j + \frac{2\Delta_{n_j}}{3}\mid\mathcal{F}_t, \tilde{x}_t = \chi_n) \cdot \mathbb{P}(\tilde{\theta}_{n,i} \leq \chi_n^\top \theta_j + \frac{2\Delta_{n_j}}{3}, \forall i \neq j, \alpha\mid\mathcal{F}_t,\tilde{x}_t = \chi_n).
\end{align*}

Finally, combining these two bounds, we arrive at:
\begin{align*}
&\mathbb{P}(\bar{a}_t = j, E^{\mu}_{n_j}(t), E^{\theta}_{n_j}(t) \mid \mathcal{F}_t, \tilde{x}_t = \chi_n)
\leq \frac{\mathbb{P}(\tilde{\theta}_{n,\alpha} \leq \chi_n^\top \theta_j + \frac{2\Delta_{n,j}}{3}\mid\mathcal{F}_t, \tilde{x}_t = \chi_n)}{\mathbb{P}(\tilde{\theta}_{n,\alpha} > \chi_n^\top \theta_j + \frac{2\Delta_{n_j}}{3}\mid\mathcal{F}_t, \tilde{x}_t = \chi_n)} \cdot \mathbb{P}(\bar{a}_t = \alpha \mid  E^{\mu}_{n_j}(t), E^{\theta}_{n_j}(t),\mathcal{F}_t, \tilde{x}_t = \chi_n).
\end{align*}

This completes the proof.
\end{proof}
By Lemma \ref{lemmaut11}, our objective of $\sum_{t=1}^T\mathbb{E}\left[\mathbf{1}(\bar{a}_t = j, E^{\mu}_{n,j}(t), E^{\theta}_{n,j}(t))|\tilde{x}_t=\chi_n,M_n(t)=t-1\right]$ can be upper bounded by:
   \begin{align*}
    &\sum_{t=1}^T\mathbb{E}\left[\mathbf{1}(\bar{a}_t = j, E^{\mu}_{n,j}(t), E^{\theta}_{n,j}(t))|\tilde{x}_t=\chi_n,M_n(t)=t-1\right]\nonumber\\
    =&\sum_{t=1}^T\mathbb{E}\left[\mathbb{P}(\bar{a}_t = j, E^{\mu}_{n,j}(t), E^{\theta}_{n,j}(t)|\tilde{x}_t=\chi_n,\mathcal{F}_t)|M_n(t)=t-1\right]\nonumber\\
    \leq&\sum_{t=1}^T\mathbb{E}\left[\frac{\mathbb{P}(\tilde{\theta}_{n,\alpha} \leq \chi_n^\top \theta_j + \frac{2\Delta_{n,j}}{3}\mid\tilde{x}_t=\chi_n,\mathcal{F}_t)}{\mathbb{P}(\tilde{\theta}_{n,\alpha} > \chi_n^\top \theta_j + \frac{2\Delta_{n,j}}{3}\mid\tilde{x}_t=\chi_n,\mathcal{F}_t)}\right.
    \cdot \left.\mathbb{P}(\bar{a}_t = \alpha \mid E^{\mu}_{n,j}(t), E^{\theta}_{n_j}(t),\tilde{x}_t=\chi_n,\mathcal{F}_t)|M_n(t)=t-1\right].\nonumber\\
    \end{align*}
The last expression can then be rewritten as
    \begin{align}
    &\sum_{t=1}^T\mathbb{E}\left[\mathbb{E}\left[\frac{\mathbb{P}(\tilde{\theta}_{n,\alpha} \leq \chi_n^\top \theta_j + \frac{2\Delta_{n,j}}{3}\mid\mathcal{F}_t, \tilde{x}_t = \chi_n)}{\mathbb{P}(\tilde{\theta}_{n,\alpha} > \chi_n^\top \theta_j + \frac{2\Delta_{n,j}}{3}\mid\mathcal{F}_t, \tilde{x}_t = \chi_n)}\right.\right. \cdot \left.\left.\mathbf{1}(\bar{a}_t = \alpha) \mid E^{\mu}_{n,j}(t), E^{\theta}_{n_j}(t),\mathcal{F}_t\right]|M_n(t)=t-1, \tilde{x}_t = \chi_n\right]\nonumber\\
    =&\sum_{t=1}^T\mathbb{E}\left[\frac{\mathbb{P}(\tilde{\theta}_{n,\alpha} \leq \chi_n^\top \theta_j + \frac{2\Delta_{n,j}}{3}\mid\mathcal{F}_t, \tilde{x}_t = \chi_n)}{\mathbb{P}(\tilde{\theta}_{n,\alpha} > \chi_n^\top \theta_j + \frac{2\Delta_{n,j}}{3}\mid\mathcal{F}_t, \tilde{x}_t = \chi_n)}\right.\cdot \left.\mathbf{1}(\bar{a}_t = \alpha )\mid E^{\mu}_{n,j}(t), E^{\theta}_{n_j}(t),M_n(t)=t-1, \tilde{x}_t = \chi_n\right]\nonumber\\
    \leq &\sum_{t=1}^T\mathbb{E}\left[\frac{\mathbb{P}(\tilde{\theta}_{n,\alpha} \leq \chi_n^\top \theta_j + \frac{2\Delta_{n,j}}{3}\mid\mathcal{F}_t, M_{n,\alpha}(t)=t-1)}{\mathbb{P}(\tilde{\theta}_{n,\alpha} >\chi_n^\top \theta_j + \frac{2\Delta_{n,j}}{3}\mid\mathcal{F}_t, M_{n,\alpha}(t)=t-1)}\right],\label{eq3}
\end{align} 
where $M_{n,j}(t)$ is the number of pulls of any arm $j\in[K]$ by context $\chi_n$ before time $t$. Recall that $\tilde{\theta}_{n,\alpha}$ is a sample from porsteior distribution $\mathcal{N}(\chi_n^\top\hat{\theta}^t_{\alpha}, ||\chi_n||_{\hat{V}^t_{\alpha}}^2)$. Then $\mathbb{P}(\tilde{\theta}_{n,\alpha} \leq \chi_n^\top \theta_j + \frac{2\Delta_{n,j}}{3}\mid\mathcal{F}_t, \tilde{x}_t = \chi_n,M_{n,\alpha}(t)=t-1)$ can be rewritten as
\begin{align*}
    \mathbb{P}\bigg(\tilde{\theta}_{n,\alpha} \leq \chi_n^\top \theta_j + \frac{2\Delta_{n,j}}{3}\mid\mathcal{F}_t,M_{n,\alpha}(t)=t-1\bigg)
    =&\mathbb{P}\bigg(\frac{\tilde{\theta}_{n,\alpha}-\chi_n^\top\hat{\theta}_\alpha}{||\chi_n||_{\hat{V}^t_\alpha}} \leq \frac{\chi_n^\top \theta_j-\chi_n^\top\hat{\theta}_\alpha+ \frac{2\Delta_{n,j}}{3}}{||\chi_n||_{\hat{V}^t_\alpha}}\mid\mathcal{F}_t, M_{n,\alpha}(t)=t-1\bigg)\\
    =&\Phi\bigg( \frac{\chi_n^\top \theta_j-\chi_n^\top\hat{\theta}_\alpha+ \frac{2\Delta_{n,j}}{3}}{||\chi_n||_{\hat{V}^t_\alpha}}\bigg),
\end{align*}
where the $\Phi(\cdot)$ represents the cumulative density function of standard normal distribution and $||\chi_n||_{\hat{V}^t_\alpha}=\sqrt{\frac{\chi_n^\top V_\alpha\chi_n}{1+(t-1)\chi_n^\top V_\alpha\chi_n}}$ when $M_{n,\alpha}(t)=t-1$ by Lemma \ref{lemmaantimean}. Substituting the final expression and the probability density function of \( \chi_n^\top \hat{\theta}^t_{\alpha} \) from Lemma \ref{lemmaantimean} into Eq. (\ref{eq3}), we obtain:
\begin{align}
    &\mathbb{E}\left[\frac{\mathbb{P}(\tilde{\theta}_{n,\alpha} \leq \chi_n^\top \theta_j + \frac{2\Delta_{n,j}}{3}\mid\mathcal{F}_t, x_t = \chi_n,M_{n,\alpha}(t)=t-1)}{\mathbb{P}(\tilde{\theta}_{n,\alpha} > \chi_n^\top \theta_j + \frac{2\Delta_{n,j}}{3}\mid\mathcal{F}_t, x_t = \chi_n,M_{n,\alpha}(t)=t-1)}\right]\nonumber\\
    =&\int^\infty_{-\infty} \frac{1}{\sqrt{2\pi}||\chi_n||_{\hat{\Sigma}^t_\alpha}}\exp\bigg(-\frac{\bigg(y-\chi_n^\top\theta_\alpha-\frac{\chi_n^\top(\mu_\alpha - \theta_\alpha)}{\chi_n^\top V_\alpha \chi_n} \|\chi_n\|^2_{\hat{V}_\alpha^t}\bigg)^2}{2||\chi_n||^2_{\hat{\Sigma}_\alpha^t}}\bigg)
    \cdot\frac{\Phi( \frac{\chi_n^\top \theta_j-y+ {2\Delta_{n,j}}/{3}}{||\chi_n||_{\hat{V}^t_\alpha}})}{1-\Phi( \frac{\chi_n^\top \theta_j-y+ {2\Delta_{n,j}}/{3}}{||\chi_n||_{\hat{V}^t_\alpha}})}dy\nonumber\\
    =&\int^\infty_{-\infty} \frac{||\chi_n||_{\hat{V}_\alpha^t}}{\sqrt{2\pi}||\chi_n||_{\hat{\Sigma}^t_\alpha}}\exp\bigg(-\frac{(||\chi_n||_{\hat{V}_\alpha^t}y+\Delta_{n,j}/3+v||\chi_n||^2_{\hat{V}_\alpha^t})^2}{2||\chi_n||^2_{\hat{\Sigma}_\alpha^t}}\bigg)
    \cdot\frac{\Phi(y)}{1-\Phi(y)}dy\nonumber\\
    \leq &\int^\infty_{0} \frac{||\chi_n||_{\hat{V}_\alpha^t}}{\sqrt{2\pi}||\chi_n||_{\hat{\Sigma}^t_\alpha}}\exp\bigg(-\frac{(||\chi_n||_{\hat{V}_\alpha^t}y+\Delta_{n,j}/3+v||\chi_n||^2_{\hat{V}_\alpha^t})^2}{2||\chi_n||^2_{\hat{\Sigma}_\alpha^t}}\bigg)
    \frac{1}{1-\Phi(y)}dy\nonumber\\
    &+2\int_{-\infty}^0 \frac{||\chi_n||_{\hat{V}_\alpha^t}}{\sqrt{2\pi}||\chi_n||_{\hat{\Sigma}^t_\alpha}}
    \exp\bigg(-\frac{(||\chi_n||_{\hat{V}_\alpha^t}y+\Delta_{n,j}/3+v||\chi_n||^2_{\hat{V}_\alpha^t})^2}{2||\chi_n||^2_{\hat{\Sigma}_\alpha^t}}\bigg)\Phi(y)dy,\label{eq4}
\end{align}
where we denote $v=\frac{\chi_n^\top(\mu_\alpha - \theta_\alpha)}{\chi_n^\top V_\alpha \chi_n} $. Next, we first examine the first term in the final expression of (\ref{eq4}):
\begin{align*}
    &\int^\infty_{0} \frac{||\chi_n||_{\hat{V}_\alpha^t}}{\sqrt{2\pi}||\chi_n||_{\hat{\Sigma}^t_\alpha}}\exp\bigg(-\frac{(||\chi_n||_{\hat{V}_\alpha^t}y+\Delta_{n,j}/3+v||\chi_n||^2_{\hat{V}_\alpha^t})^2}{2||\chi_n||^2_{\hat{\Sigma}_\alpha^t}}\bigg)\frac{1}{1-\Phi(y)}dy\\
    \leq &\int^\infty_{0} \frac{||\chi_n||_{\hat{V}_\alpha^t}}{\sqrt{2\pi}||\chi_n||_{\hat{\Sigma}^t_\alpha}}\exp\bigg(-\frac{(||\chi_n||_{\hat{V}_\alpha^t}y+\Delta_{n,j}/3+v||\chi_n||^2_{\hat{V}_\alpha^t})^2}{2||\chi_n||^2_{\hat{\Sigma}_\alpha^t}}\bigg) \sqrt{\frac{\pi}{2}}\exp(y^2/2) \bigg(y+\sqrt{y^2+8/\pi}\bigg)    dy\\
    \leq&\int^\infty_{0} \frac{||\chi_n||_{\hat{V}_\alpha^t}}{||\chi_n||_{\hat{\Sigma}^t_\alpha}}\exp\bigg(-\frac{(||\chi_n||_{\hat{V}_\alpha^t}y+\Delta_{n,j}/3+v||\chi_n||^2_{\hat{V}_\alpha^t})^2}{2||\chi_n||^2_{\hat{\Sigma}_\alpha^t}}\bigg)\exp(y^2/2) \bigg(y+1\bigg)    dy,
\end{align*}
where the second inequality follows from Lemma \ref{fact1}. With some algebraic manipulation, the final integral can be derived as:
\begin{align*}    
    &\frac{||\chi_n||_{\hat{V}_\alpha^t}||\chi_n||_{\hat{\Sigma}_\alpha^t}}{||\chi_n||_{\hat{V}_\alpha^t}^2-||\chi_n||^2_{\hat{\Sigma}_\alpha^t}}\exp\bigg(-\frac{(\Delta_{n,j}/3+||\chi_n||^2_{\hat{V}_\alpha^t}v)^2}{2||\chi_n||^2_{\hat{\Sigma}_\alpha^t}}\bigg)\\
    &+\int^\infty_{\frac{||\chi_n||_{\hat{V}_\alpha^t}(\Delta_{n,j}/3+||\chi_n||^2_{\hat{V}_\alpha^t}v)}{||\chi_n||_{\hat{\Sigma}_\alpha^t}\sqrt{||\chi_n||^2_{\hat{V}_\alpha^t}-||\chi_n||^2_{\hat{\Sigma}_\alpha^t}}}}\exp\bigg(-\frac{y^2}{2}+\frac{(\Delta_{n,j}/3+||\chi_n||^2_{\hat{V}_\alpha^t}v)^2}{2(||\chi_n||^2_{\hat{V}_\alpha^t}-||\chi_n||^2_{\hat{\Sigma}_\alpha^t})}\bigg)dy\bigg(1-\frac{||\chi_n||_{\hat{V}_\alpha^t}(\Delta_{n,j}/3+||\chi_n||^2_{\hat{V}_\alpha^t}v)}{||\chi_n||^2_{\hat{V}_\alpha^t}-||\chi_n||^2_{\hat{\Sigma}_\alpha^t}}\bigg)\frac{||\chi_n||_{\hat{V}_\alpha^t}}{\sqrt{||\chi_n||^2_{\hat{V}_\alpha^t}-||\chi_n||^2_{\hat{\Sigma}_\alpha^t}}}.
\end{align*}
Substituting $||\chi_n||^2_{\hat{\Sigma}_\alpha^t}=(t-1)\bigg(\frac{\chi_n^\top V_\alpha\chi_n}{1+(t-1)\chi_n^\top V_\alpha\chi_n}\bigg)^2$ and $||\chi_n||^2_{\hat{V}_\alpha^t}=\frac{\chi_n^\top V_\alpha\chi_n}{1+(t-1)\chi_n^\top V_\alpha\chi_n}$ into this expression, we obtain
\begin{align}    
    &\sqrt{(t-1)^2(\chi_n^\top V_\alpha\chi_n)^2+1}\exp\bigg(-\frac{t-1}{2}\bigg(\frac{\Delta_{n,j}}{3}\bigg(1+\frac{1}{(t-1)\chi_n^\top V_\alpha \chi_n}\bigg)+\frac{v}{t-1}\bigg)^2\bigg)\nonumber\\
    &+\int^\infty_{\frac{\Delta_{n,j}/3+||\chi_n||^2_{\hat{V}_\alpha^t}v}{||\chi_n||^3_{\hat{V}_\alpha^t}}\sqrt{\frac{\chi_n^\top V_\alpha\chi_n}{t-1}}}\exp\bigg(-\frac{y^2}{2}+\frac{(\Delta_{n,j}/3+||\chi_n||^2_{\hat{V}_\alpha^t}v)^2}{2(||\chi_n||^2_{\hat{V}_\alpha^t}-||\chi_n||^2_{\hat{\Sigma}_\alpha^t})}\bigg)dy\bigg(1-(1+(t-1)\chi_n^\top V_\alpha\chi_n)\bigg(\frac{\Delta_{n,j}/3}{||\chi_n||_{\hat{V}_\alpha^t}}+v||\chi_n||_{\hat{V}_\alpha^t}\bigg)\bigg)\sqrt{1+(t-1)\chi_n^\top V_\alpha\chi_n}.\label{eq6}
\end{align}
First, the expression above must be finite. Then, as \( t \) increases, the first term in the expression (\ref{eq6}) converges exponentially. For the second term, we can show that there exists \( t_1 \) such that when $t>t_1$, it will become negative. This implies that summing the expression above from \( t = 1 \) to \( t=T\) yields a constant value, summarized as \(C_{n,j}^{1,1}\).

Next we further derive the second term in the last line of Eq. (\ref{eq4}). For the term \(\frac{\Delta_{n,j}}{3} + v||\chi_n||^2_{\hat{V}_\alpha^t}\) where \(v=\frac{\chi_n^\top(\mu_\alpha-\theta_\alpha)}{\chi_n^\top V_\alpha\chi_n}\), \(v\) can be either positive or negative depending on the realization of the prior. However, there must exist a time \(t_2\) such that \(\frac{\Delta_{n,j}}{3} + v||\chi_n||^2_{\hat{V}_\alpha^t} \geq 0\) for all \(t > t_2\). When \(t > t_2\), the second term simplifies to:
\begin{align}
&2\int_{-\infty}^0 \frac{||\chi_n||_{\hat{V}_\alpha^t}}{\sqrt{2\pi}||\chi_n||_{\hat{\Sigma}^t_\alpha}}\exp\bigg(-\frac{(||\chi_n||_{\hat{V}_\alpha^t}y+\Delta_{n,j}/3+v||\chi_n||^2_{\hat{V}_\alpha^t})^2}{2||\chi_n||^2_{\hat{\Sigma}_\alpha^t}}\bigg)\int^y_{-\infty}\frac{\exp(-t^2/2)}{\sqrt{2\pi}}dtdy\nonumber\\
\leq&2\int_{-\infty}^\infty \frac{||\chi_n||_{\hat{V}_\alpha^t}}{\sqrt{2\pi}||\chi_n||_{\hat{\Sigma}^t_\alpha}}\exp\bigg(-\frac{(||\chi_n||_{\hat{V}_\alpha^t}y+\Delta_{n,j}/3+v||\chi_n||^2_{\hat{V}_\alpha^t})^2}{2||\chi_n||^2_{\hat{\Sigma}_\alpha^t}}\bigg)\int^y_{-\infty}\frac{\exp(-t^2/2)}{\sqrt{2\pi}}dtdy\nonumber\\
=&2\int_{-\infty}^0 \frac{||\chi_n||_{\hat{V}_\alpha^t}}{\sqrt{2\pi}\sqrt{||\chi_n||_{\hat{V}_\alpha^t}+||\chi_n||_{\hat{\Sigma}_\alpha^t}}}\exp\bigg(-\frac{(z||\chi_n||_{\hat{V}_\alpha^t} - {\Delta_{n,j}}/{3}-v||\chi_n||^2_{\hat{V}_\alpha^t})^2}{2||\chi_n||_{\hat{V}_\alpha^t}+2||\chi_n||_{\hat{\Sigma}_\alpha^t}}\bigg)dz\nonumber\\
\leq &\exp\bigg(-\frac{( {\Delta_{n,j}}/{3}+v||\chi_n||^2_{\hat{V}_\alpha^t})^2}{2||\chi_n||_{\hat{V}_\alpha^t}+2||\chi_n||_{\hat{\Sigma}_\alpha^t}}\bigg).\label{eq5}
\end{align}
Therefore, summing the second term in Eq. (\ref{eq4}) from \( t = 1 \) to \( t=T\) also yields a constant value, summarized as \(C_{n,j}^{1,2}\). Further denote $C_{n,j}^{1}=C_{n,j}^{1,1}+C_{n,j}^{1,2}$ for the summation from $t=1$ to $T$ of the upper bounds in (\ref{eq6}) and (\ref{eq5}), we obtain the result in Lemma 5.2 and complete the proof.

\section{Proof of Lemma \ref{lemmaut2}}\label{APlemmaut2}

To prove Lemma \ref{lemmaut2}, decompose the left-hand side of the inequality in Lemma \ref{lemmaut2} into two parts with \(M_{n,j}(t) \leq z\) and \(M_{n,j}(t) > z\) as follows:
\begin{align*}
    &\sum_{t=1}^T\mathbb{E}[\mathbf{1}(\bar{a}_t=\alpha_n,\tilde{\theta}^t_{n,j}-\chi_n^T\hat{\theta}_j^t>\delta, M_{n,j}(t)\leq z)+\mathbf{1}(\bar{a}_t=\alpha_n,\tilde{\theta}^t_{n,j}-\chi_n^T\hat{\theta}_j^t>\delta, M_{n,j}(t)> z)\mid \tilde{x}_t=\chi_n,M_n(t)=t-1]\\
    \leq &z+\sum_{t=1}^T\mathbb{E}[\mathbf{1}(\bar{a}_t=j,\tilde{\theta}^t_{n,j}-\chi_n^T\hat{\theta}_j^t>\delta, M_{n,j}(t)> z)\mid \tilde{x}_t=\chi_n,M_n(t)=t-1]\\
     \leq &z+\sum_{t=1}^T\mathbb{E}[\mathbf{1}(\tilde{\theta}^t_{n,j}-\chi_n^T\hat{\theta}_j^t>\delta)\mid M_{n,j}(t)> z, \tilde{x}_t=\chi_n,M_n(t)=t-1]\\
    \leq &z+\sum_{t=1}^T\mathbb{E}\bigg[\sqrt{\frac{2}{\pi}}\frac{\exp\bigg(-\frac{\delta^2}{2||\chi_n||^2_{\hat{V}_j^{t}}}\bigg)}{\delta/||\chi_n||_{\hat{V}_j^{t}}+\sqrt{\delta^2/||\chi_n||^2_{\hat{V}_j^{t}}+\frac{8}{\pi}}}\mid M_{n,j}(t)> z,\tilde{x}_t=\chi_n \bigg]\\
    \leq &z+\frac{T}{2}\exp\bigg(-\frac{\delta^2}{2}\bigg(\frac{1}{\chi_n^TV_j\chi_n}+z\bigg)\bigg),
\end{align*}
where the third inequality follows from the concentration inequality in Lemma \ref{fact1}, and the last inequality is obtained by using \(||\chi_n||^2_{\hat{V}_j^{t}} = \frac{\chi_n^\top V_j \chi_n}{1+(t-1)\chi_n^\top V_j \chi_n}\) from Lemma \ref{lemmaantimean}. Choosing \(z = \frac{2}{\delta^2} \ln \frac{T \delta^2}{4} - \frac{1}{\chi_n^\top V_j \chi_n}\) above, we obtain:
\begin{align*}
    &\sum_{t=1}^T\mathbb{E}\bigg[\mathbf{1}(\bar{a}_t=j,\tilde{\theta}^t_{n,j}-\chi_n^T\hat{\theta}_j^t>\delta)\mid \tilde{x}_t=\chi_n,M_n(t)=t-1\bigg]
    \leq \frac{2}{\delta^2}\ln \frac{T\delta^2}{4}-\frac{1}{\chi_n^TV_j\chi_n}+\frac{2}{\delta^2}.
\end{align*}
This completes the proof.

We also present the corollary of Lemma \ref{lemmaut2} to be used in the proof of Lemma \ref{opt_less}:

\begin{corollary}\label{lemmaut22}
In Thompson sampling, the expected number of pulls of suboptimal arm $j\neq \alpha$ under any context $\chi_n\in\mathcal{X}$, together with the occurrence of event $|\tilde{\theta}^t_{n,j}-\chi_n^\top\hat{\theta}^t_j|>\delta$, can be upper bounded as follows:
\begin{align*}
    &\sum_{t=1}^T\mathbb{E}\bigg[\mathbf{1}(\bar{a}_t=j,|\tilde{\theta}^t_{n,j}-\chi_n^T\hat{\theta}_j^t|>\delta)\mid \tilde{x}_t=\chi_n,M_n(t)=t-1\bigg]
    \leq \frac{2}{\delta^2}\ln \frac{T\delta^2}{2}-\frac{1}{\chi_n^TV_j\chi_n}+\frac{2}{\delta^2}
\end{align*}
\end{corollary}

\section{Proof of Lemma \ref{lemmaut3}}\label{APlemmaut3}
To prove Lemma \ref{lemmaut3}, choose \(L = \max\left(\left\lceil \frac{\chi_n^\top (\mu_j - \theta_j)}{\delta \chi_n^\top V_j \chi_n} - \frac{1}{\chi_n^\top V_j \chi_n} \right\rceil, 0\right)\) such that for \(t-1 \geq L\), we can use (\ref{eq7}) in Lemma \ref{lemmaantimean} to upper bound the probability of \( \chi_n^\top \hat{\theta}_j^t - \chi_n^\top \theta_j > \delta \). We then divide the process into two parts with \(t \leq L+1\) and \(t > L+1\), and apply Lemma \ref{lemmaantimean} in the fourth inequality below :
\begin{align*}
    &\sum_{t=1}^T\mathbb{E}\bigg[\mathbf{1}(\bar{a}_t=j,\chi_n^T\hat{\theta}_j^t-\chi_n^\top\theta_{j}>\delta)\mid \tilde{x}_t=\chi_n,M_n(t)=t-1\bigg]\\
    \leq &\sum_{t=1}^T\mathbb{E}\bigg[\mathbf{1}(\chi_n^T\hat{\theta}_j^t-\chi_n^\top\theta_{j}>\delta)\mid \tilde{x}_t=\chi_n,M_{n,j}(t)=t-1\bigg]\\
    = &\sum_{t=1}^{L+1}\mathbb{E}\bigg[\mathbf{1}(\chi_n^T\hat{\theta}_j^t-\chi_n^\top\theta_{j}>\delta)\mid \tilde{x}_t=\chi_n,M_{n,j}(t)=t-1\bigg]+\sum_{t=L+2}^T\mathbb{E}\bigg[\mathbf{1}(\chi_n^T\hat{\theta}_j^t-\chi_n^\top\theta_{j}>\delta)\mid \tilde{x}_t=\chi_n,M_{n,j}(t)=t-1\bigg]\\
    \leq&L+1+\frac{1}{2}\sum_{t=L+2}^{T}\exp\bigg(-\bigg(\frac{\delta}{||\chi_n||_{\hat{V}_j^{t}}}-\frac{\chi_n^T(\mu_j-\theta_j)||\chi_n||_{\hat{V}_j^{t}}}{\chi_n^TV_j\chi_n}\bigg)^2/2\bigg)\\
     \leq&L+1+\frac{1}{2}\int^\infty_{L+1}\exp\bigg(-\bigg(\frac{\delta}{||\chi_n||_{\hat{V}_j^{t}}}-\frac{\chi_n^T(\mu_j-\theta_j)||\chi_n||_{\hat{V}_j^{t}}}{\chi_n^TV_j\chi_n}\bigg)^2/2\bigg)dt\\
    = &L+1+\frac{\exp\bigg(-\frac{\delta^2L}{2}-\frac{\delta^2}{2\chi_n^TV_j\chi_n}+\frac{\delta\chi_n^T(\mu_j-\theta_j)}{\chi_n^TV_j\chi_n}\bigg)}{\delta^2}\\
    \leq &\max\left(\left\lceil \frac{\chi_n^\top (\mu_j - \theta_j)}{\delta \chi_n^\top V_j \chi_n} - \frac{1}{\chi_n^\top V_j \chi_n} \right\rceil, 0\right)+1+\frac{1}{\delta^2}{\exp\bigg(\frac{\delta\chi_n^T(\mu_j-\theta_j)}{2\chi_n^TV_j\chi_n}\bigg)},
\end{align*}
where the last equality follows from $\hat{V}^t_j=\sqrt{\frac{\chi_n^\top V_j\chi_n}{1+(t-1)\chi_n^\top V_j\chi_n}}$. The following corollary of Lemma \ref{lemmaut3} will be used for the proof of Lemma \ref{opt_less}:
\begin{corollary}\label{lemmaut33}
In Thompson sampling, the expected number of pulls of suboptimal arm $j\neq \alpha$ under any context $\chi_n\in\mathcal{X}$, together with the occurrence of event $|\chi_n^\top\hat{\theta}^t_j-\chi_n^\top\hat{\theta}^t_j|>\delta$, can be upper bounded as follows:
    \begin{align*}
        &\sum_{t=1}^T\mathbb{E}\bigg[\mathbf{1}(\bar{a}_t=j,|\chi_n^T\hat{\theta}_j^t-\chi_n^\top\theta_{j}|>\delta)\mid \tilde{x}_t=\chi_n,M_n(t)=t-1\bigg]
        \leq\max\left(\left\lceil \frac{\chi_n^\top (\mu_j - \theta_j)}{\delta \chi_n^\top V_j \chi_n} - \frac{1}{\chi_n^\top V_j \chi_n} \right\rceil, 0\right)+1+\frac{2}{\delta^2}\exp\bigg(\frac{\delta\chi_n^T(\mu_j-\theta_j)}{2\chi_n^TV_j\chi_n}\bigg).
    \end{align*}
\end{corollary}

\section{Proof of Lemma \ref{conflictupper}}\label{APconflictupper}

Recall the definition of \(I(t)\):
\begin{align*}
    I(t): \chi_1^\top\Theta^t(\mathbf{p}^t(\chi_2)-\mathbf{p}^t(\chi_1))>0 \text{ or }\chi_2^\top\Theta^t(\mathbf{p}^t(\chi_1)-\mathbf{p}^t(\chi_2))>0.
\end{align*}
We decompose $\chi_n^T\Theta^t(\mathbf{p}^t(\chi_n)-\mathbf{p}^t(\chi_{3-n}))$ as follows:
    \begin{align*}
        &\chi_n^T\Theta^t(\mathbf{p}^t(\chi_n)-\mathbf{p}^t(\chi_{3-n}))\\
        =&\sum_{k\in[K]}\chi_n^T(\hat{\theta}_{k}^t-\hat{\theta}_{\underline{k_n}}^t)({p}^t_k(\chi_n)-p^t_k(\chi_{3-n}))\\
        =&\chi_n^T(\hat{\theta}_{\alpha_n}^t-\hat{\theta}_{\underline{k_n}}^t)({p}^t_{\alpha_n}(\chi_n)-p^t_{\alpha_n}(\chi_{3-n}))+\chi_n^T(\hat{\theta}_{\alpha_{3-n}}^t-\hat{\theta}_{\underline{k_n}}^t)({p}^t_{\alpha_{3-n}}(\chi_n)-p^t_{\alpha_{3-n}}(\chi_{3-n}))+\sum_{k\neq \alpha_1,\alpha_2}\chi_n^T(\hat{\theta}_{k}^t-\hat{\theta}_{\underline{k_n}}^t)({p}^t_k(\chi_n)-p^t_k(\chi_{3-n}))\\
        \geq&\chi_n^T(\hat{\theta}_{\alpha_n}^t-\hat{\theta}_{\underline{k_n}}^t)({p}^t_{\alpha_n}(\chi_n)-p^t_{\alpha_n}(\chi_{3-n}))+\chi_n^T(\hat{\theta}_{\alpha_{3-n}}^t-\hat{\theta}_{\underline{k_n}}^t)({p}^t_{\alpha_{3-n}}(\chi_n)-p^t_{\alpha_{3-n}}(\chi_{3-n}))+\sum_{k\neq \alpha_1,\alpha_2}\chi_n^T(\hat{\theta}_{k}^t-\hat{\theta}_{\underline{k_n}}^t)(-{p}^t_k(\chi_n)-p^t_k(\chi_{3-n}))\\
        \geq&\chi_n^T(\hat{\theta}_{\alpha_n}^t-\hat{\theta}_{\underline{k_n}}^t)({p}^t_{\alpha_n}(\chi_n)-p^t_{\alpha_n}(\chi_{3-n}))+\chi_n^T(\hat{\theta}_{\alpha_{3-n}}^t-\hat{\theta}_{\underline{k_n}}^t)({p}^t_{\alpha_{3-n}}(\chi_n)-p^t_{\alpha_{3-n}}(\chi_{3-n}))+\max_{k\neq \alpha_1,\alpha_2}\chi_n^T(\hat{\theta}_{k}^t-\hat{\theta}_{\underline{k_n}}^t)\sum_{k\neq \alpha_1,\alpha_2}(-{p}^t_k(\chi_n)-p^t_k(\chi_{3-n}))
        \\
        =&\chi_n^T(\hat{\theta}_{\alpha_n}^t-\hat{\theta}_{\underline{k_n}}^t)({p}^t_{\alpha_n}(\chi_n)-p^t_{\alpha_n}(\chi_{3-n}))+\chi_n^T(\hat{\theta}_{\alpha_{3-n}}^t-\hat{\theta}_{\underline{k_n}}^t)({p}^t_{\alpha_{3-n}}(\chi_n)-p^t_{\alpha_{3-n}}(\chi_{3-n}))-\max_{k\neq \alpha_1,\alpha_2}\chi_n^T(\hat{\theta}_{k}^t-\hat{\theta}_{\underline{k_n}}^t)(2-p^t_{\alpha_n}(\chi_n)-p^t_{\alpha_{3-n}}(\chi_{3-n})).
\end{align*}
In the last line, applying \(p^t_{\alpha_n}(\chi_{3-n}) < 1 - p^t_{\alpha_{3-n}}(\chi_{3-n})\) and \({p}^t_{\alpha_{3-n}}(\chi_n) - p^t_{\alpha_{3-n}}(\chi_{3-n}) > -1\) to the first two terms, we obtain the following lower bound for the last line:
\begin{align*}        
        \chi_n^T(\hat{\theta}_{\alpha_n}^t-\hat{\theta}_{\underline{k_n}}^t)({p}^t_{\alpha_n}(\chi_n)+p^t_{\alpha_{3-n}}(\chi_{3-n})-1)-\chi_n^T(\hat{\theta}_{\alpha_{3-n}}^t-\hat{\theta}_{\underline{k_n}}^t)-\max_{k\neq \alpha_1,\alpha_2}\chi_n^T(\hat{\theta}_{k}^t-\hat{\theta}_{\underline{k_n}}^t)(2-{p}^t_{\alpha_n}(\chi_n)-p^t_{\alpha_{3-n}}(\chi_{3-n}))
    \end{align*}
Therefore, $0>\chi_n^T\Theta^t(\mathbf{p}^t(\chi_n)-\mathbf{p}^t(\chi_{3-n}))$ implies that the above expression must also be less than 0, which is:
    \begin{align*}
        &{p}^t_{\alpha_n}(\chi_n)+p^t_{\alpha_{3-n}}(\chi_{3-n})
        \leq2-\frac{\chi_1^\top\hat{\theta}^t_{\alpha_n}-\chi_1^\top\hat{\theta}^t_{\alpha_{3-n}}}{\chi_1^\top\hat{\theta}_{\alpha_n}^t-\chi_1^\top\hat{\theta}^t_{{\underline{k_n}}}+\max_{k\neq \alpha_1,\alpha_2}\chi_n^T(\hat{\theta}_{k}^t-\hat{\theta}_{\underline{k_n}}^t)}
    \end{align*}
Combining this condition for both $\chi_1$ and $\chi_2$ yields the result.

\section{Proof of Lemma \ref{opt_less}}\label{APopt_less}
We first rewrite the complete version of Lemma \ref{opt_less} as follows:

\begin{lemma}\label{opt_less-}
Let $\epsilon^t$ denote the following expression:
 \begin{align*}
     \min_n\frac{\chi_n^\top\hat{\theta}^t_{\alpha_n}-\chi_n^\top\hat{\theta}^t_{\alpha_{3-n}}}{2(\chi_n^\top\hat{\theta}_{\alpha_n}^t-\chi_n^\top\hat{\theta}^t_{{\underline{k_n}}}+\max_{k\neq \alpha_1,\alpha_2}\chi_n^T(\hat{\theta}_{k_n}^t-\hat{\theta}_{\underline{k_n}}^t))}.
 \end{align*}
Then the expected number of pulls together with the occurrence of event ${p^t_{\alpha_{n}}(\chi_{n})<1-\epsilon^t}$ is upper bounded by:
        \begin{align*}
&\sum_{t=1}^T\mathbb{E}[\mathbf{1}(x_t=\chi_{n},\tilde{x}_t=\chi_{n},\bar{a}_t =\alpha_{n},{p^t_{\alpha_{n}}(\chi_{n})<1-\epsilon^t})]\\
\leq&\frac{2048}{\Delta_{n}^2}\bigg(\ln \frac{T\Delta_{n}^2}{2048}+1+\exp\bigg(\frac{\Delta_n(\mu_\alpha-\theta_\alpha)}{64\chi_n^\top V_\alpha\chi_n}\bigg)\bigg)-\frac{1}{\chi_n^TV_{\alpha_n}\chi_n}+\max\bigg(\frac{32\chi_n^\top(\mu_\alpha-\theta_\alpha)}{\Delta_n\chi_n^\top V_\alpha\chi_n}-\frac{1}{\chi_n^\top V_\alpha\chi_n},0\bigg)+D_{n,i}^1+D'+1
    \end{align*}
    where $D^1_{n,i}$ is a constant given in Eq. (\ref{eq13}), and $D'$ is a constant.
\end{lemma}

To start the proof of Lemma \ref{opt_less-}, we define the following four events:
        \begin{align}
        &\hat{E}^{\mu}_{n,\alpha_n}(t):|\chi_n^T\hat{\theta}_{\alpha_n}^t-\chi_n^T{\theta}_{\alpha_n}|<\frac{\Delta_n}{32},\hat{E}^{\theta}_{n,\alpha_n}(t):|\tilde{\theta}_{n,\alpha_n}^t-\chi_n^T\hat{\theta}_{\alpha_n}^t|<\frac{\Delta_n}{32}\nonumber\\
        &\hat{E}^{\mu}_{n,i}(t):|\chi_n^T\hat{\theta}_{i}^t-\chi_n^T{\theta}_{i}|<\frac{9\Delta_{n,i}}{16},\hat{E}^{\theta}_{n,i}(t):\tilde{\theta}_{n,i}^t-\chi_n^T{\theta}_{i}<\frac{15\Delta_{n,i}}{16},\label{events}
    \end{align}
    where $\Delta_n=\min_{i\in[K]}\Delta_{n,i}$. Denote the event of \(\cup_{i}\hat{E}^\mu_{n,i}(t)\cup\hat{E}^\theta_{n,i}(t)\) as \(Y_n(t)\). We then decompose $\mathbb{E}[\mathbf{1}(x_t=\chi_{n},\tilde{x}_t=\chi_{n},\bar{a}_t =\alpha_{n},{p^t_{\alpha_{n}}(\chi_{n})<1-\epsilon^t})]$ as follows:
        \begin{align}
&\mathbb{E}[\mathbf{1}(x_t=\chi_{n},\tilde{x}_t=\chi_{n},\bar{a}_t =\alpha_{n},{p^t_{\alpha_{n}}(\chi_{n})<1-\epsilon^t})]\nonumber\\
= &\mathbb{E}[\mathbf{1}(x_t=\chi_{n},\tilde{x}_t=\chi_{n},\bar{a}_t =\alpha_{n},{p^t_{\alpha_{n}}(\chi_{n})<1-\epsilon^t})\bigg(\mathbf{1}(\overline{\hat{E}^\mu_{n,\alpha_{n}}(t)})+\mathbf{1}({\hat{E}^\mu_{n,\alpha_{n}}(t)},\overline{\hat{E}^\theta_{n,\alpha_{n}}(t)})+\mathbf{1}({\hat{E}^\mu_{n,\alpha_{n}}(t)},{\hat{E}^\theta_{n,\alpha_{n}}(t)},\cup_{i\neq \alpha_n}\overline{\hat{E}^\mu_{n,i}(t)})\nonumber\\
&+\mathbf{1}({\hat{E}^\mu_{n,\alpha_{n}}(t)},{\hat{E}^\theta_{n,\alpha_{n}}(t)},\cap_{i\neq \alpha_n}{\hat{E}^\mu_{n,i}(t)})\bigg)]\nonumber\\
\leq &\mathbb{E}[\mathbf{1}(x_t=\chi_{n},\tilde{x}_t=\chi_{n},\bar{a}_t =\alpha_{n},\overline{\hat{E}^\mu_{n,\alpha_{n}}(t)})]+\mathbb{E}[\mathbf{1}(x_t=\chi_{n},\tilde{x}_t=\chi_{n},\bar{a}_t =\alpha_{n},{\hat{E}^\mu_{n,\alpha_{n}}(t)},\overline{\hat{E}^\theta_{n,\alpha_{n}}(t)})]\nonumber\\
&+\sum_{i\in[K]}\mathbb{E}[\mathbf{1}(x_t=\chi_{n},\tilde{x}_t=\chi_{n},\bar{a}_t =\alpha_{n},\hat{E}^\mu_{n,\alpha_{n}}(t),\hat{E}^\theta_{n,\alpha_{n}}(t),\overline{\hat{E}^\mu_{n,i}(t)})]\nonumber\\
&+\mathbb{E}[\mathbf{1}(x_t=\chi_{n},\tilde{x}_t=\chi_{n},\bar{a}_t =\alpha_{n},\hat{E}^\mu_{n,\alpha_{n}}(t),\cap_{i\neq \alpha_n}{\hat{E}^\mu_{n,i}(t)},{p^t_{\alpha_{n}}(\chi_{n})<1-\epsilon^t})]\label{eq14}
    \end{align}

Using \(\delta = \Delta_{n}/32\) in Corollaries \ref{lemmaut22} and \ref{lemmaut33},  the summation from $t=1$ to $T$ of the first two terms in the last expression in (\ref{eq14}) is upper bounded by:
\begin{align}
\frac{2048}{\Delta_{n}^2}\bigg(\ln \frac{T\Delta_{n}^2}{2048}+1+\exp\bigg(\frac{\Delta_n(\mu_\alpha-\theta_\alpha)}{64\chi_n^\top V_\alpha\chi_n}\bigg)\bigg)-\frac{1}{\chi_n^TV_{\alpha_n}\chi_n}+\max\bigg(\frac{32\chi_n^\top(\mu_\alpha-\theta_\alpha)}{\Delta_n\chi_n^\top V_\alpha\chi_n}-\frac{1}{\chi_n^\top V_\alpha\chi_n},0\bigg)+1.\label{eq19}
\end{align}
Then we upper bound the remaining terms in Eq. (\ref{eq14}). We introduce the following lemma to aid in proving the third term, which converts the terms with \(\bar{a}_t = \alpha_n\) to terms with \(\bar{a}_t = i\).
\begin{lemma}
\label{lemmaut111}
When the events \( \hat{E}^{\mu}_{n,\alpha_n}(t) \) and \( \hat{E}^{\theta}_{n,\alpha_n}(t) \) happen, the probabilies of selecting the optimal arm \( \alpha_n \) and the suboptimal arm \( i \) for context \( \chi_n \) in Thompson sampling satisfy the following inequality:
\begin{align*}
    &\mathbb{P}(\tilde{x}_t=\chi_n,\bar{a}_t=\alpha_n,\hat{E}^{\mu}_{n,\alpha_n}(t), \hat{E}^{\theta}_{n,\alpha_n}(t)\mid\mathcal{F}_t)
    \leq \frac{\mathbb{P}(\tilde{\theta}_{n,i}^t < \chi_n^\top \theta_{\alpha_n} + \frac{\Delta_n}{16} \mid \mathcal{F}_t)}{\mathbb{P}(\tilde{\theta}_{n,i}^t \geq \chi_n^\top \theta_{\alpha_n} + \frac{\Delta_n}{16} \mid \mathcal{F}_t)} \mathbb{P}(\tilde{x}_t=\chi_n,\bar{a}_t=i,\hat{E}^{\mu}_{n,\alpha_n}(t), \hat{E}^{\theta}_{n,\alpha_n}(t)\mid\mathcal{F}_t).
\end{align*}
\end{lemma}
\begin{proof}
When events \( \hat{E}^{\mu}_{n,\alpha_n}(t) \) and \( \hat{E}^{\theta}_{n,\alpha_n}(t) \) happen, the Thompson sampling value of arm \( \alpha_n \) falls within the interval \( (\chi_n^\top \theta_{\alpha_n} - \frac{\Delta_n}{16}, \chi_n^\top \theta_{\alpha_n} + \frac{\Delta_n}{16})\). Therefore, selecting arm \( \alpha_n \) through Thompson sampling together with events \( \hat{E}^{\mu}_{n,\alpha_n}(t) \) and \( \hat{E}^{\theta}_{n,\alpha_n}(t) \) implies that no other arm $j$ has a Thompson sampling value $\tilde{\theta}^t_{n,j}$ exceeding \( \chi_n^\top \theta_{\alpha_n} + \frac{\Delta_n}{16} \) at time $t$, which leads to the following:
\begin{align*}
    &\mathbb{P}(\tilde{x}_t = \chi_n, \bar{a}_t = \alpha_n, \hat{E}^{\mu}_{n,\alpha_n}(t), \hat{E}^{\theta}_{n,\alpha_n}(t) \mid \mathcal{F}_t)\\
    \leq& \mathbb{P}(\tilde{x}_t = \chi_n, \tilde{\theta}_{n,i}^t < \chi_n^\top \theta_{\alpha_n} + \frac{\Delta_n}{16},\forall i \neq \alpha_n,\hat{E}^{\mu}_{n,\alpha_n}(t), \hat{E}^{\theta}_{n,\alpha_n}(t) \mid \mathcal{F}_t)\\
    =& \mathbb{P}(\tilde{\theta}_{n,i}^t < \chi_n^\top \theta_{\alpha_n} + \frac{\Delta_n}{16} \mid \tilde{x}_t = \chi_n,\mathcal{F}_t)\mathbb{P}( \tilde{\theta}_{n,j}^t < \chi_n^\top \theta_{\alpha_n} + \frac{\Delta_n}{16},\forall j \neq \alpha_n,i \mid \tilde{x}_t = \chi_n,\mathcal{F}_t) \mathbb{P}(\tilde{x}_t = \chi_n,\hat{E}^{\mu}_{n,\alpha_n}(t), \hat{E}^{\theta}_{n,\alpha_n}(t) \mid \mathcal{F}_t),
\end{align*}
where the last equality holds because the Thompson sampling values for any arms are mutually independent, conditional on the history \(\mathcal{F}_t\). Conversely, if the Thompson sampling value of a suboptimal arm \( i \) exceeds \( \chi_n^\top \theta_{\alpha_n} + \frac{\Delta_n}{16} \), while the values for all other arms remain below \( \chi_n^\top \theta_{\alpha_n} + \frac{\Delta_n}{16} \), the Thompson sampling algorithm will select arm \( i \). Therefore, we have:
\begin{align*}
    &\mathbb{P}(\tilde{x}_t = \chi_n, \bar{a}_t = i , \hat{E}^{\mu}_{n,\alpha_n}(t), \hat{E}^{\theta}_{n,\alpha_n}(t)\mid\mathcal{F}_t) \\
    >& \mathbb{P}(\tilde{\theta}_{n,i}^t \geq \chi_n^\top \theta_{\alpha_n} + \frac{\Delta_n}{16}\mid \tilde{x}_t = \chi_n,\mathcal{F}_t)\mathbb{P}(\tilde{\theta}_{n,j}^t < \chi_n^\top \theta_{\alpha_n} + \frac{\Delta_n}{16} ,\forall j \neq \alpha_n,i \mid \tilde{x}_t = \chi_n,\mathcal{F}_t)  \mathbb{P}(\tilde{x}_t = \chi_n, \hat{E}^{\mu}_{n,\alpha_n}(t), \hat{E}^{\theta}_{n,\alpha_n}(t)\mid\mathcal{F}_t)
\end{align*}

Combining these two inequalities then we complete the proof of Lemma \ref{lemmaut111}.

\end{proof}

For the third term in (\ref{eq14}), although they include additional parts \(\hat{E}^{\mu}_{n,\alpha_n}(t)\) and \({\hat{E}^{\mu}_{n,i}(t)},\overline{\hat{E}^{\theta}_{n,i}(t)})\) compared to Lemma \ref{lemmaut111}, following the same procedure as in the proof of Lemma \ref{lemmaut111} still yields similar upper bounds. We first proceed with the third term in (\ref{eq14}). The upper bound of the third term in (\ref{eq14}) are upper bounded by (\ref{eq11}) as follows:
\begin{align}
    &\mathbb{E}[\mathbb{E}[\mathbf{1}({x}_t=\chi_{n},\tilde{x}_t=\chi_{n},\bar{a}_t=\alpha_n,\hat{E}^{\mu}_{n,\alpha_n}(t),{\hat{E}^{\theta}_{n,\alpha_n}(t)},\overline{\hat{E}^{\mu}_{n,i}(t)})|\mathcal{F}_t]]\nonumber\\
    \leq &\mathbb{E}\bigg[\frac{\mathbb{P}(\tilde{\theta}_{n,i}^t<\chi_n^T\theta_{\alpha_n}+\frac{\Delta_n}{16},\overline{\hat{E}^{\mu}_{n,i}(t)}|\mathcal{F}_t)}{\mathbb{P}(\tilde{\theta}_{n,i}^t\geq\chi_n^T\theta_{\alpha_n}+\frac{\Delta_n}{16},\overline{\hat{E}^{\mu}_{n,i}(t)}|\mathcal{F}_t)}\mathbb{P}({x}_t=\chi_{n},\tilde{x}_t=\chi_{n},\bar{a}_t=i,\hat{E}^{\mu}_{n,\alpha_n}(t),{\hat{E}^{\theta}_{n,\alpha_n}(t)},\overline{\hat{E}^{\mu}_{n,i}(t)}|\mathcal{F}_t)\bigg].\label{eq11}
\end{align}
Summing up from $t=1$ to $T$, we continue to proceed with Eq. (\ref{eq11}) as follows:
\begin{align}
    &\sum_{t=1}^T\mathbb{E}\bigg[\frac{\mathbb{P}(\tilde{\theta}_{n,i}^t<\chi_n^T\theta_{\alpha_n}+\frac{\Delta_n}{16},\overline{\hat{E}^{\mu}_{n,i}(t)}|\mathcal{F}_t)}{\mathbb{P}(\tilde{\theta}_{n,i}^t\geq\chi_n^T\theta_{\alpha_n}+\frac{\Delta_n}{16},\overline{\hat{E}^{\mu}_{n,i}(t)}|\mathcal{F}_t)}\mathbb{P}(x_t=\chi_{n},\tilde{x}_t=\chi_{n},\bar{a}_t=i,\hat{E}^{\mu}_{n,\alpha_n}(t),{\hat{E}^{\theta}_{n,\alpha_n}(t)},\overline{\hat{E}^{\mu}_{n,i}(t)}|\mathcal{F}_t)\bigg]\nonumber\\
    \leq&\sum_{t=1}^T\mathbb{E}\bigg[\frac{\mathbb{P}(\tilde{\theta}_{n,i}^t<\chi_n^T\theta_{\alpha_n}+\frac{\Delta_n}{16}|\overline{\hat{E}^{\mu}_{n,i}(t)},\mathcal{F}_t)}{\mathbb{P}(\tilde{\theta}_{n,i}^t\geq\chi_n^T\theta_{\alpha_n}+\frac{\Delta_n}{16}|\overline{\hat{E}^{\mu}_{n,i}(t)},\mathcal{F}_t)}\mathbb{P}(x_t=\chi_{n},\tilde{x}_t=\chi_{n},\bar{a}_t=i,\overline{\hat{E}^{\mu}_{n,i}(t)}|\mathcal{F}_t)\bigg]\nonumber\\
    \leq&\sum_{t=1}^T\mathbb{E}\bigg[\frac{\mathbb{P}(x_t=\chi_{n},\tilde{x}_t=\chi_{n},\bar{a}_t=i,\overline{\hat{E}^{\mu}_{n,i}(t)}|\mathcal{F}_t)}{\mathbb{P}(\tilde{\theta}_{n,i}^t\geq\chi_n^T\theta_{\alpha_n}+\frac{\Delta_n}{16}|\overline{\hat{E}^{\mu}_{n,i}(t)},\mathcal{F}_t)}\bigg]\nonumber\\
        \leq&\sum_{t=1}^T\mathbb{E}\bigg[\frac{\mathbb{P}(x_t=\chi_{n},\tilde{x}_t=\chi_{n},\bar{a}_t=i,\overline{\hat{E}^{\mu}_{n,i}(t)}|\mathcal{F}_t)}{\frac{1}{2}\mathbb{P}(\tilde{\theta}_{n,i}^t\geq\chi_n^T\hat{\theta}_{i}^t+\frac{\Delta_{n,i}}{2}|\mathcal{F}_t)}\bigg]\nonumber\\
    \leq&\sum_{t=1}^T\mathbb{E}\bigg[\frac{2\mathbb{P}(\overline{\hat{E}^{\mu}_{n,i}(t)}|M_{n,i}(t)=t-1,\mathcal{F}_t)}{\mathbb{P}(\tilde{\theta}_{n,i}^t\geq\chi_n^T\hat{\theta}_{i}^t+\frac{\Delta_{n,i}}{2}|M_{n,i}(t)=t-1,\mathcal{F}_t)}\bigg].\label{eq9-}
\end{align}
The second last inequality holds because, conditioned on event \(\overline{\hat{E}^{\mu}_{n,i}(t)}\), there is a probability of \(1/2\) that \(\overline{\hat{E}^{\mu}_{n,i}(t)}\) occurs with \(\chi_n^\top \hat{\theta}_i^t > \chi_n^\top {\theta}_i + \frac{9\Delta_{n,i}}{16}\). Thus, when \(\tilde{\theta}_{n,i}^t \geq \chi_n^\top \hat{\theta}_i^t + \frac{\Delta_{n,i}}{2}\), we have:
\[
    \tilde{\theta}_{n,i}^t \geq \chi_n^\top \hat{\theta}_i^t + \frac{\Delta_{n,i}}{2} > \chi_n^\top \theta_i + \frac{9\Delta_{n,i}}{16} + \frac{\Delta_{n,i}}{2} = \chi_n^\top \theta_{\alpha_n} + \frac{\Delta_{n,i}}{16}\geq \chi_n^\top \theta_{\alpha_n} + \frac{\Delta_{n}}{16},
\]
which implies that
\begin{align*}
    \mathbb{P}(\tilde{\theta}_{n,i}^t\geq\chi_n^T\theta_{\alpha_n}+\frac{\Delta_{n}}{16}|\overline{\hat{E}^{\mu}_{n,i}(t)},\mathcal{F}_t) 
    > \frac{1}{2} \mathbb{P}(\tilde{\theta}_{n,i}^t \geq \chi_n^\top \hat{\theta}_i^t + \frac{\Delta_{n,i}}{2} \mid \overline{\hat{E}^{\mu}_{n,i}(t)},\mathcal{F}_t)
    =\frac{1}{2} \mathbb{P}(\tilde{\theta}_{n,i}^t \geq \chi_n^\top \hat{\theta}_i^t + \frac{\Delta_{n,i}}{2} \mid \mathcal{F}_t).
\end{align*}
To preceed with (\ref{eq9-}), we choose \(L = \max\left(\left\lceil \frac{16\chi_n^\top (\mu_i - \theta_i)}{9\Delta_{n,i}\chi_n^\top V_i \chi_n} - \frac{1}{\chi_n^\top V_i \chi_n} \right\rceil, 0\right)\) such that for $t-1\geq L$, we can use Eq. (\ref{eq7}) to upper bound the probability of $\hat{E}^{\mu}_{n,i}(t)$ where $|\chi_n^\top\hat{\theta}^t_j-\chi_n^\top{\theta}_j|>\frac{9\Delta_{n,i}}{16}$. We then divide the process into two parts with $t\leq L$ and $t>L$, and apply the concentration inequality (\ref{eq7}) in Lemma \ref{lemmaantimean} to the denominator of (\ref{eq9-}) and the anti-concentration inequality in Lemma \ref{fact1} to the numerator of (\ref{eq9-}) to upper bound (\ref{eq9-}) as follows:
\begin{align}
&\sum_{t=L+1}^{T-1}\frac{4\exp(-\frac{1}{2}(\frac{9\Delta_{n,i}}{16||\chi_n||_{\hat{V}_i^t}}-\frac{\chi_n^T(\mu_i-\theta_i)||\chi_n||_{\hat{V}_i^t}}{\chi_n^TV_i\chi_n})^2)}{\exp(-\frac{1}{2}(\frac{\Delta_{n,i}}{2||\chi_n||_{\hat{V}_i^t}})^2)}\frac{\frac{\Delta_{n,i}}{2||\chi_n||_{\hat{V}_i^{t}}}+\sqrt{(\frac{\Delta_{n,i}}{2||\chi_n||_{\hat{V}_i^{t}}})^2+4}}{\frac{9\Delta_{n,i}}{16||\chi_n||_{\hat{V}_i^{t}}}-\frac{\chi_n^T(\mu_i-\theta_i)||\chi_n||_{\hat{V}_i^{t}}}{\chi_n^TV_i\chi_n}+\sqrt{(\frac{9\Delta_{n,i}}{16||\chi_n||_{\hat{V}_i^{t}}}-\frac{\chi_n^T(\mu_i-\theta_i)||\chi_n||_{\hat{V}_i^{t}}}{\chi_n^TV_i\chi_n})^2+8/\pi}}\nonumber\\
&+\sum_{t=0}^{L}\frac{\frac{\Delta_{n,i}}{2||\chi_n||_{\hat{V}_i^{t}}}+\sqrt{(\frac{\Delta_{n,i}}{2||\chi_n||_{\hat{V}_i^t}})^2+8/\pi}}{\exp(-\frac{1}{2}(\frac{\Delta_{n,i}}{2||\chi_n||_{\hat{V}_i^t}})^2)}\nonumber\\
=&\sum_{t=L+1}^{T-1}4\exp\bigg(-\frac{17\Delta_{n,i}}{512||\chi_n||^2_{\hat{V}_i^t}}-\frac{9\Delta_{n,i}\chi_n^T(\mu_i-\theta_i)}{16\chi_n^TV_i\chi_n}+\bigg(\frac{\chi_n^T(\mu_i-\theta_i)||\chi_n||_{\hat{V}_i^t}}{\chi_n^TV_i\chi_n}\bigg)^2\bigg)\nonumber\\
&\cdot\frac{\frac{\Delta_{n,i}}{2||\chi_n||_{\hat{V}_i^{t}}}+\sqrt{(\frac{\Delta_{n,i}}{2||\chi_n||_{\hat{V}_i^{t}}})^2+4}}{\frac{9\Delta_{n,i}}{16||\chi_n||_{\hat{V}_i^{t}}}-\frac{\chi_n^T(\mu_i-\theta_i)||\chi_n||_{\hat{V}_i^{t}}}{\chi_n^TV_i\chi_n}+\sqrt{(\frac{9\Delta_{n,i}}{16||\chi_n||_{\hat{V}_i^{t}}}-\frac{\chi_n^T(\mu_i-\theta_i)||\chi_n||_{\hat{V}_i^{t}}}{\chi_n^TV_i\chi_n})^2+8/\pi}}+\sum_{t=0}^{L}\frac{\frac{\Delta_{n,i}}{2||\chi_n||_{\hat{V}_i^{t}}}+\sqrt{(\frac{\Delta_{n,i}}{2||\chi_n||_{\hat{V}_i^t}})^2+8/\pi}}{\exp(-\frac{1}{2}(\frac{\Delta_{n,i}}{2||\chi_n||_{\hat{V}_i^t}})^2)}.\label{eq13}
\end{align}
Since \(\frac{1}{||\chi_n||_{\hat{V}_i^{t}}^2} = \frac{1}{\chi_n^\top V_i \chi_n }+ (t-1)\) by Lemma \ref{lemmaantimean}, the first term in the first line of (\ref{eq13}) decays exponentially with \(t\), while the first term in the second line converges to a constant, and the last term is a constant. Thus, the entire equation in (\ref{eq13}) last a constant, summarized by \(D^1_{n,i}\).

For the forth term in (\ref{eq14}), we upper bound it in the following lemma.
\begin{lemma}
The expected number of time steps at which all of the following events occur is bounded above by the following term under some constant $D'$:
    \begin{align}
        \sum_{t=1}^T\mathbb{E}[\mathbf{1}(x_t=\chi_{n},\tilde{x}_t=\chi_{n},\bar{a}_t =\alpha_{n},\hat{E}^\mu_{n,\alpha_{n}}(t),\cap_{i\neq \alpha_n}{\hat{E}^\mu_{n,i}(t)},{p^t_{\alpha_{n}}(\chi_{n})<1-\epsilon^t})]\leq D'.\label{eq557}
    \end{align}
\end{lemma}
\begin{proof}
    We first decompose the event as follows:
    \begin{align*}
        &\sum_{t=1}^T\mathbb{E}[\mathbf{1}(x_t=\chi_{n},\tilde{x}_t=\chi_{n},\bar{a}_t =\alpha_{n},\cap_{i\in [K]}{\hat{E}^\mu_{n,i}(t)},{p^t_{\alpha_{n}}(\chi_{n})<1-\epsilon^t})]\\
        =&\sum_{t=1}^T\mathbb{E}[\mathbf{1}(x_t=\chi_{n},\tilde{x}_t=\chi_{n},\bar{a}_t =\alpha_{n},\cap_{i\in [K]}{\hat{E}^\mu_{n,i}(t)},{\sum_{i\neq \alpha_n}p^t_{i}(\chi_{n})\geq\epsilon^t},{p^t_{\alpha_{n}}(\chi_{n})<1-\epsilon^t})]\\
        \leq &\sum_{t=1}^T\sum_{i\neq \alpha_n}\mathbb{E}[\mathbf{1}(x_t=\chi_{n},\tilde{x}_t=\chi_{n},\bar{a}_t =\alpha_{n},\cap_{i\in [K]}{\hat{E}^\mu_{n,i}(t)},{\mathbb{P}(\bar{a}_t=i|x_t=\chi_{n},\tilde{x}_t=\chi_{n},\mathcal{F}_t)\geq\epsilon^t/(K-1)},{p^t_{\alpha_{n}}(\chi_{n})<1-\epsilon^t})].\\
    \end{align*}
    Under event $\cap_{i\in [K]}{\hat{E}^\mu_{n,i}(t)}$, if the Thomspon sampling algorithm choose a suboptimal arm $i$, at least one of events $\tilde{\theta}^t_{n,i}\geq \chi_n^T\hat{\theta^t_i}+3\Delta_{n,i}/8$ and $\overline{\hat{E}_{n,\alpha_n}^\theta(t)}$ must occur, which implies that 
    \begin{align*}
        \mathbb{P}(\{\tilde{\theta}^t_{n,i}\geq \chi_n^T\hat{\theta^t_i}+3\Delta_{n,i}/8\}\cup\overline{\hat{E}_{n,\alpha_n}^\theta(t))}\geq\mathbb{P}(\bar{a}_t=i|x_t=\chi_{n},\tilde{x}_t=\chi_{n},\mathcal{F}_t).
    \end{align*}
    Therefore,
    \begin{align}
        &\sum_{t=1}^T\sum_{i\neq \alpha_n}\mathbb{E}[\mathbf{1}(x_t=\chi_{n},\tilde{x}_t=\chi_{n},\bar{a}_t =\alpha_{n},\cap_{i\in [K]}{\hat{E}^\mu_{n,i}(t)},{\mathbb{P}(\bar{a}_t=i|x_t=\chi_{n},\tilde{x}_t=\chi_{n},\mathcal{F}_t)\geq{\epsilon}^t/(K-1)},{p^t_{\alpha_{n}}(\chi_{n})<1-\epsilon^t})]\nonumber\\
        \leq&\sum_{t=1}^T\sum_{i\neq \alpha_n}\mathbb{E}[\mathbf{1}(x_t=\chi_{n},\tilde{x}_t=\chi_{n},\bar{a}_t =\alpha_{n},\cap_{i\in [K]}{\hat{E}^\mu_{n,i}(t)},{\mathbb{P}(\{\tilde{\theta}^t_{n,i}\geq \chi_n^T\hat{\theta^t_i}+3\Delta_{n,i}/8\}\cup\overline{\hat{E}_{n,\alpha_n}^\theta(t))})\geq{\epsilon}^t/(K-1)},{p^t_{\alpha_{n}}(\chi_{n})<1-\epsilon^t})]\nonumber\\
        \leq&\sum_{t=1}^T\sum_{i\neq \alpha_n}\mathbb{E}[\mathbf{1}(x_t=\chi_{n},\tilde{x}_t=\chi_{n},\bar{a}_t =\alpha_{n},\cap_{i\in [K]}{\hat{E}^\mu_{n,i}(t)},{\mathbb{P}(\{\tilde{\theta}^t_{n,i}\geq \chi_n^T\hat{\theta^t_i}+3\Delta_{n,i}/8\})\geq{\epsilon}^t/(2(K-1))},{p^t_{\alpha_{n}}(\chi_{n})<1-\epsilon^t})]\nonumber\\
        &+\sum_{t=1}^T\sum_{i\neq \alpha_n}\mathbb{E}[\mathbf{1}(x_t=\chi_{n},\tilde{x}_t=\chi_{n},\bar{a}_t =\alpha_{n},\cap_{i\in [K]}{\hat{E}^\mu_{n,i}(t)},{\mathbb{P}(\overline{\hat{E}_{n,\alpha_n}^\theta(t))})\geq{\epsilon}^t/(2(K-1))},{p^t_{\alpha_{n}}(\chi_{n})<1-\epsilon^t})]\nonumber\\
        \leq&\sum_{t=1}^T\frac{1-{\epsilon}^t}{{\epsilon}^t/(2(K-1))}\sum_{i\neq \alpha_n}\mathbb{E}[\mathbf{1}(x_t=\chi_{n},\tilde{x}_t=\chi_{n},\bar{a}_t =i,\cap_{i\in [K]}{\hat{E}^\mu_{n,i}(t)},{\mathbb{P}(\{\tilde{\theta}^t_{n,i}\geq \chi_n^T\hat{\theta^t_i}+3\Delta_{n,i}/8\})\geq\bar{\epsilon}/(2(K-1))})]\nonumber\\
        &+\sum_{t=1}^T\sum_{i\neq \alpha_n}\mathbb{E}[\mathbf{1}(x_t=\chi_{n},\tilde{x}_t=\chi_{n},\bar{a}_t =\alpha_{n},\cap_{i\in [K]}{\hat{E}^\mu_{n,i}(t)},{\mathbb{P}(\overline{\hat{E}_{n,\alpha_n}^\theta(t))})\geq{\epsilon}^t/(2(K-1))})].\label{eq556}
    \end{align}
    Note that when event $\cap_i\hat{E}^\mu_{n,i}(t)$ happens, the $\varepsilon^t$ in (\ref{trutherror}) have a stictly positive lower bound $\bar{\varepsilon}$ that depends on $\Delta_n$, thus the $\frac{1-{\epsilon}^t}{{\epsilon}^t/(2(K-1))}$ in the last expression is upper bounded by a constant $\frac{1-\bar{\epsilon}}{\bar{\epsilon}/(2(K-1))}$. 
    We will next show that, when summing up from $t=1$ to $T$, each of the $\sum_{t=1}^T\mathbb{E}[\mathbf{1}(x_t=\chi_{n},\tilde{x}_t=\chi_{n},\bar{a}_t =\alpha_{n},\cap_{i\in [K]}{\hat{E}^\mu_{n,i}(t)},{\mathbb{P}(\hat{E}_{n,\alpha_n}^\theta(t))\geq\bar{\epsilon}/(2(K-1))})]$ and $\sum_{t=1}^T\mathbb{E}[\mathbf{1}(x_t=\chi_{n},\tilde{x}_t=\chi_{n},\bar{a}_t =i,\cap_{i\in [K]}{\hat{E}^\mu_{n,i}(t)},{\mathbb{P}(\{\tilde{\theta}^t_{n,i}\geq \chi_n^T\hat{\theta^t_i}+3\Delta_{n,i}/8\})\geq\bar{\epsilon}/(2(K-1))},{p^t_{\alpha_{n}}(\chi_{n})<1-\epsilon^t})]$ can be upper bounded by a constant. We derive the constant upper bound for the former case, and the latter case can be upper bounded similarly. By Lemma \ref{fact1}, we have
    \begin{align*}
        \mathbb{P}(\{\tilde{\theta}^t_{n,i}\geq \chi_n^T\hat{\theta^t_i}+3\Delta_{n,i}/8\})\leq {\frac{1}{2}}{e^{-\bigg(\frac{3\Delta_{n,i}}{8\|\chi_n\|_{\hat{V}_i^t}}\bigg)^2}}{}.
    \end{align*}
    Then,
    \begin{align*}
        &\sum_{t=1}^T\mathbb{E}[\mathbf{1}(x_t=\chi_{n},\tilde{x}_t=\chi_{n},\bar{a}_t =\alpha_{n},\cap_{i\in [K]}{\hat{E}^\mu_{n,i}(t)},{{\frac{1}{2}}{e^{-\bigg(\frac{3\Delta_{n,i}}{8\|\chi_n\|\hat{V}_i^t}\bigg)^2}}{}\geq\bar{\epsilon}/(2(K-1))})]\\
        =&\sum_{t=1}^T\mathbb{E}[\mathbf{1}(x_t=\chi_{n},\tilde{x}_t=\chi_{n},\bar{a}_t =\alpha_{n},\cap_{i\in [K]}{\hat{E}^\mu_{n,i}(t)},{{\frac{1}{2}}{e^{-\bigg(\frac{3\Delta_{n,i}}{8}\bigg)^2\bigg(\frac{1}{\chi_n^\top V_j \chi_n}+|\mathcal{T}_{n,i}^t|-1\bigg)}}{}\geq\bar{\epsilon}/(2(K-1))})]\\
        =&\sum_{t=1}^T\mathbb{E}[\mathbf{1}(x_t=\chi_{n},\tilde{x}_t=\chi_{n},\bar{a}_t =\alpha_{n},\cap_{i\in [K]}{\hat{E}^\mu_{n,i}(t)},|\mathcal{T}_{n,i}^t|<\frac{64}{9\Delta_{n,i}^2}\ln\frac{K-1}{\bar{\epsilon}}+1-\frac{1}{\chi_n^TV_i\chi_i})]\\
        \leq&\frac{64}{9\Delta_{n,i}^2}\ln\frac{K-1}{\bar{\epsilon}}+1-\frac{1}{\chi_n^TV_i\chi_i},
    \end{align*}
    where $|\mathcal{T}_{n,i}^t|$ represent the times that the system choose arm $i$ under context $\chi_n$ up to time $t$, and the first equation is derived by Lemma \ref{lemmaantimean}. Based on the upper bound above, the complete upper bound in (\ref{eq556}) yields a constant, then we complete the proof.
\end{proof}

By upper bounding the summation from $t=1$ to $T$ for each term in \((\ref{eq14})\) by (\ref{eq19}), (\ref{eq13}), and (\ref{eq557}), we derive the results presented in Lemma \ref{opt_less-}. Furthermore, by denoting all constant components in Lemma \ref{opt_less-} as \(D_n\), we arrive at Lemma \ref{opt_less}, thus completing the proof.



\end{document}